\documentclass{article}
\usepackage[paperheight=11in,paperwidth=8.5in]{geometry}

\usepackage{authblk}
\usepackage{microtype}
\usepackage{graphicx}
\usepackage{caption}
\usepackage{subcaption}
\usepackage{multirow}
\usepackage{booktabs} 
\usepackage{wrapfig}
\usepackage{hyperref}
\usepackage{nccmath}
\usepackage{comment}
\usepackage{algorithm,algorithmic,amsmath}
\PassOptionsToPackage{numbers}{natbib}

\usepackage{natbib}

\usepackage[utf8]{inputenc} 
\usepackage[T1]{fontenc}    
\usepackage{hyperref}       
\usepackage{url}            
\usepackage{booktabs}       
\usepackage{amsfonts}       
\usepackage{nicefrac}       
\usepackage{microtype}      
\usepackage{xcolor}         

\usepackage{amsmath}
\usepackage{amssymb}
\usepackage{mathtools}
\usepackage{amsthm,balance}

\usepackage[capitalize,noabbrev]{cleveref}

\theoremstyle{plain}
\newtheorem{theorem}{Theorem}[section]
\newtheorem{proposition}[theorem]{Proposition}
\newtheorem{lemma}[theorem]{Lemma}

\theoremstyle{definition}

\theoremstyle{remark}

\DeclareMathOperator{\tr}{tr}
\DeclareMathOperator{\adj}{adj}

\author[1]{Xingzi Xu\thanks{xingzi.xu@duke.edu}}
\author[2]{Ali Hasan\thanks{ali.hasan@duke.edu}}
\author[1]{Khalil Elkhalil\thanks{khalil.elkhalil@duke.edu}}
\author[3]{Jie Ding\thanks{dingj@umn.edu}}
\author[1]{Vahid Tarokh\thanks{vahid.tarokh@duke.edu}}

\affil[1]{Department of Electrical and Computer Engineering, Duke University}
\affil[2]{Department of Biomedical Engineering, Duke University}
\affil[3]{School of Statistics, University of Minnesota-Twin Cities}

\usepackage[textsize=tiny]{todonotes}
\date{}
\title{\textbf{Characteristic Neural Ordinary Differential Equations}}

\begin{document}
\bibliographystyle{plainnat}



\maketitle
\begin{abstract}

We propose Characteristic-Neural Ordinary Differential Equations (C-NODEs), a framework for extending Neural Ordinary Differential Equations (NODEs) beyond ODEs.
While NODEs model the evolution of a latent variables as the solution to an ODE, C-NODE models the evolution of the latent variables as the solution of a family of first-order quasi-linear partial differential equations (PDEs) along curves on which the PDEs reduce to ODEs, referred to as characteristic curves.
This in turn allows the application of the standard frameworks for solving ODEs, namely the adjoint method.
Learning optimal characteristic curves for given tasks improves the performance and computational efficiency, compared to state of the art NODE models.
We prove that the C-NODE framework extends the classical NODE on classification tasks by demonstrating explicit C-NODE representable functions not expressible by NODEs. 
Additionally, we present C-NODE-based continuous normalizing flows, which describe the density evolution of latent variables along multiple dimensions.
Empirical results demonstrate the improvements provided by the proposed method for classification and density estimation on CIFAR-10, SVHN, and MNIST datasets under a similar computational budget as the existing NODE methods.
The results also provide empirical evidence that the learned curves improve the efficiency of the system through a lower number of parameters and function evaluations compared with baselines.

\end{abstract}

\section{Introduction}
\begin{wrapfigure}{R}{0.4\textwidth}
    \centering
    \includegraphics[width=0.4\textwidth, trim= 1cm 2cm 0cm 0cm,clip]{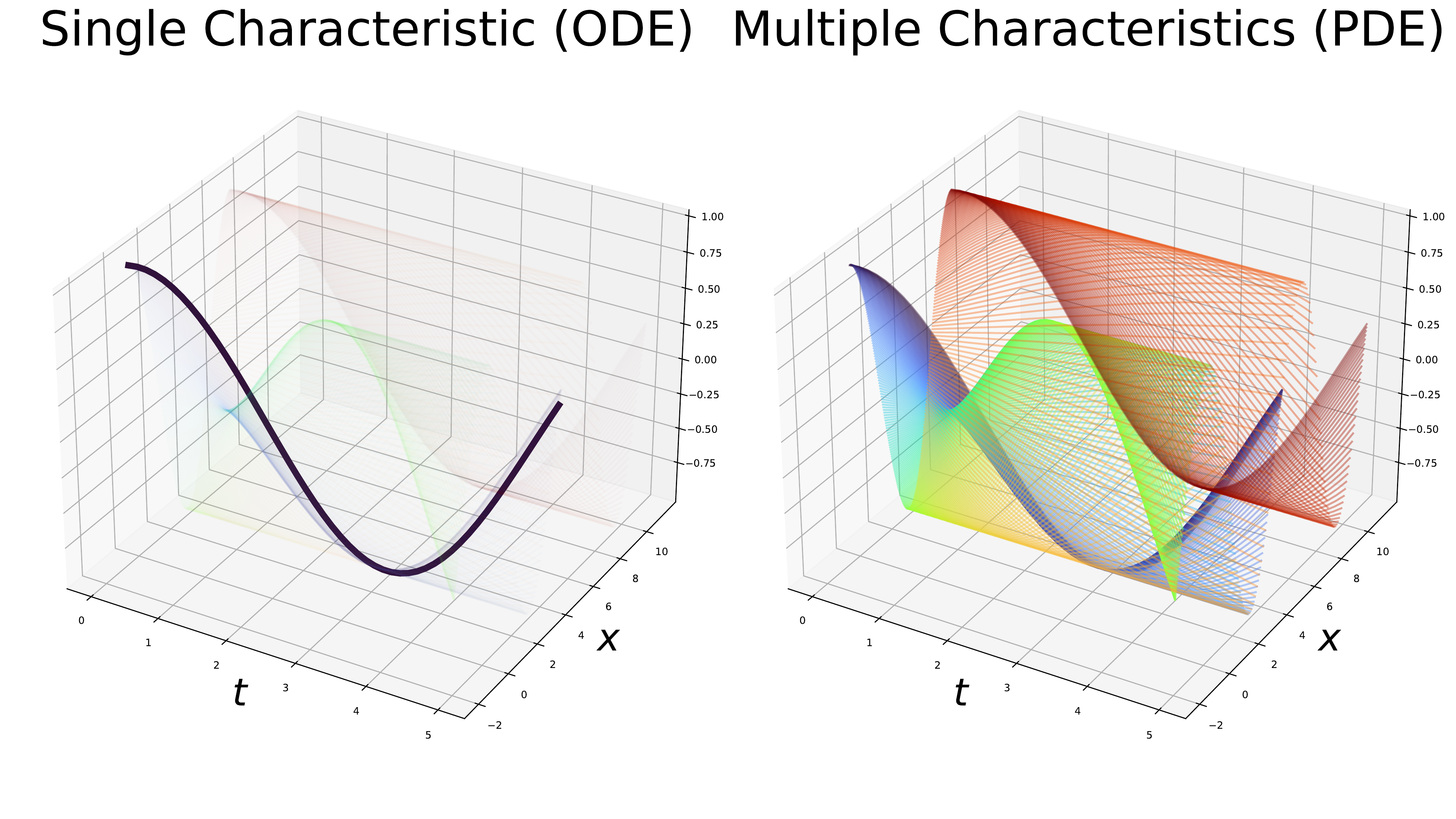}
    \caption{Comparison of traditional NODE (left) and proposed C-NODE (right). The solution to NODE is the solution to a single ODE, whereas C-NODE represents a series of ODEs that form the solution to a PDE. Each color in C-NODE represents the solution to an ODE with a different initial condition. NODE represents a single ODE, and can only represent $u(x,t)$ along one dimension, for example, $u(x=0,t)$.
    \vspace{-15pt}}
    \label{fig:motivation}
\end{wrapfigure}
Deep learning and differential equations share many connections, and techniques in the intersection have led to insights in both fields.
One predominant connection is based on certain neural network architectures resembling numerical integration schemes, leading to the development of Neural Ordinary Differential Equations (NODEs) \citep{chen2019neural}. 
NODEs use a neural network parameterization of an ODE to learn a mapping from observed variables to a latent variable that is the solution to the learned ODE.
A central benefit of NODEs is the constant memory cost, where backward passes are computed using the adjoint sensitivity method rather than backpropagating through individual forward solver steps. 
Moreover, NODEs provide a flexible probability density representation often referred to as \emph{continuous normalizing flows} (CNFs).  
However, since NODEs can only represent solutions to ODEs, the class of functions is somewhat limited and may not apply to more general problems that do not have smooth and one-to-one mappings.
To address this limitation, a series of analyses based on methods from differential equations have been employed to enhance the representation capabilities of NODEs, such as the technique of controlled differential equations \citep{kidger2020neural}, learning higher-order ODEs \citep{massaroli2021dissecting}, augmenting dynamics \citep{dupont2019augmented}, and considering dynamics with delay terms \citep{zhu2021neural}.
Moreover, certain works consider generalizing the ODE case to partial differential equations (PDEs), such as in \citet{ruthotto2020deep,sun2019neupde}.
However, these methods do not use the adjoint method, removing the primary advantage of constant memory cost. 
This leads us to the central question motivating the work: can we combine the benefits of the rich function class of PDEs with the efficiency of the adjoint method?
To do so, we propose a method of continuous-depth neural networks that solves a PDE over parametric curves that reduce the PDE to an ODE. 
Such curves are known as \emph{characteristics}, and they define the solution of the PDE in terms of an ODE \citep{griffiths2015essential}.
The proposed Characteristic Neural Ordinary Differential Equations (C-NODE) learn both the characteristics and the ODE along the characteristics to solve the PDE over the data space. 
This allows for a richer class of models while still incorporating the same memory efficiency of the adjoint method. 
The proposed C-NODE is also an extension of existing methods, as it improves the empirical accuracy of these methods in classification tasks and image quality in generation tasks.

\section{Related Work}\label{sec:related}
We discuss the related work from both a machine learning and numerical analysis perspective. 

\subsection{Machine Learning and ODEs}
NODE is often motivated as a continuous form of a Residual Network (ResNet) \citep{he2015deep}, since the ResNet can be seen as a forward Euler integration scheme on the latent state \citep{e2017dynamical}. 
Specifically, a ResNet is composed of multiple blocks where each block can be represented as:
$$u_{t+1}=u_t+f(u_t,\theta),$$
where $u_t$ is the evolving hidden state at time $t$ and $f(u_t,\theta)$ represents the gradient at time $t$, namely $\frac{du}{dt}(u_t)$. Generalizing the model to a step size given by $\Delta t$, we have:
$$u_{t+\Delta t}=u_t+f(u_t,\theta)\Delta t.$$

To adapt this model to a continuous setting, we let $\Delta t\rightarrow 0$ and obtain:
$$\lim\limits_{\Delta t \rightarrow 0}\frac{u_{t+\Delta t}-u_t}{\Delta t}=\frac{du(t)}{dt}.$$
The model can then be evaluated through existing numerical integration techniques, as proposed by \citep{chen2019neural}:
\begin{align*}
u(t_1) &=u(t_0)+\int_{t_0}^{t_1}\frac{du(t)}{dt}(u(t),t) \mathrm{d}t=u(t_0)+\int_{t_0}^{t_1}f(u(t),t,\theta) \mathrm{d}t.
\end{align*}
Numerical integration can then be treated as a black box, using numerical schemes beyond the forward Euler to achieve higher numerical precision. 
However, since black box integrators can take an arbitrary number of intermediate steps, backpropagating through individual steps would take too much memory since the individual steps must be saved. 
\citet{chen2019neural} solved this problem by using adjoint backpropagation, which has a constant memory usage. 
For a given loss function on the terminal state of the hidden state $\mathcal{L}(u(t_1))$, the adjoint $a(t)$ is governed by another ODE:
$$\frac{da(t)}{dt}=-a(t)^\intercal\frac{\partial f(u(t),t,\theta)}{\partial u}, \quad a(t_1)=\frac{\partial \mathcal{L}}{\partial u(t_1)},$$
that dictates the gradient with respect to the parameters.
The loss $\mathcal{L}(u(t_1))$ can then be calculated by solving another ODE (the adjoint) rather than backpropagating through the calculations involved in the numerical integration.

However, the hidden state governed by an ODE imposes a limitation on the expressiveness of the mapping.
For example, in \citet{dupont2019augmented}, the authors describe a notable limitation of NODEs is in the inability to represent dynamical systems with intersecting trajectories. 
In response to such limitations, many works have tried to increase the expressiveness of the mapping.
\citet{dupont2019augmented} proposed to solve the intersection trajectories problem by augmenting the vector space, lifting the points into additional dimensions; \citet{zhu2021neural} included time delay in the equation to represent dynamical systems of greater complexity; \citet{massaroli2021dissecting} proposed to condition the vector field on the inputs, allowing the integration limits to be conditioned on the input; \citet{massaroli2021dissecting} and \citet{norcliffe2020second} additionally proposed and proved a second-order ODE system can efficiently solve the intersecting trajectories problem.

Multiple works have attempted to expand NODE systems to other common differential equation formulations. \citet{sun2019neupde} employed a dictionary method and expanded NODEs to a PDE case, achieving high accuracies both in approximating PDEs and in classifying real-world image datasets. 
However, \citet{sun2019neupde} suggested that the method is unstable when training with the adjoint method and therefore is unable to make use of the benefits that come with training with adjoint. \citet{zhang2018mongeampere} proposed a normalizing flow approach based on the Monge-Ampere equation. 
However, \citet{zhang2018mongeampere} did not consider using adjoint-based training. 
\citet{long2018pde,long2019pde,raissi2019686,brunton2016discovering} considered discovering underlying hidden PDEs from data and predict dynamics of complex systems.
\citet{kidger2020neural, morrill2021neural,morrill2021rough} used ideas from rough path theory and controlled differential equations to propose a NODE architecture as a continuous recurrent neural network framework. Multiple works have expanded to the stochastic differential equations setting and developed efficient optimization methods for them \citep{guler2019robust,jia2019neural,jia2020neural,kidger2021efficient,kidger2021neural,li2020scalable,liu2019neural,xu2022infinitely}. \citet{salvi2022spde} considered stochastic PDEs for spatio-temporal dynamics prediction. Additionally, \citet{chen2020neural} models spatio-temporal data using NODEs, and \citet{rubanova2019latent,de2019gru} makes predictions on time series data using NODEs. Physical modeling is also a popular application of NODEs, as control problems are often governed by latent differential equations that can be discovered with data driven methods \citep{cranmer2020lagrangian,greydanus2019hamiltonian,pmlr-v139-yildiz21a,zhong2019symplectic}. 

NODE systems have also been used for modeling the flow from a simple probability density to a complicated one \citep{chen2019neural}. Specifically, if $u(t)\in\mathbb{R}^n$ follows the ODE $du(t)/dt=f(u(t))$, where $f(u(t))\in\mathbb{R}^n$, then its log likelihood from \citet[Appendix A]{chen2019neural}:
\begin{align}
\label{eqn:odeflow}
    \frac{\partial \log p(u(t))}{\partial t}=-\tr\left(\frac{df}{du(t)}\right)
\end{align}
The trace can be calculated efficiently with a Hutchinson trace estimator \citep{grathwohl2019ffjord}. Subsequent work uses invertible ResNet, optimal transport theory, among other techniques to further improve the performance of CNFs \citep{abdal2021styleflow,Behrmann2019InvertibleRN,chen2019ResidualFF,durkan2019neural,hoogeboom2019emerging,Huang2021ConvexPF,toth2019hamiltonian,yildiz2019ode2vae,zhang2018mongeampere}. CNF is desirable for having no constraints on the type of neural network used, unlike discrete normalizing flows, which often have constraints on the structure of the latent features \citep{Dinh2017DensityEU,Papamakarios2017MaskedAF,JimenezRezende2015VariationalIW}. CNFs also inspire development in other generative modeling methods. For instance, a score-based generative model can be seen as a probability flow modeled with an ODE \citep{Song2021ScoreBasedGM,Vahdat2021ScorebasedGM}.

\section{Method}
We describe the proposed C-NODE method in this section by first introducing the method of characteristics (MoC) for solving PDEs with an illustrative example.
We then discuss how to apply MoC to our C-NODE framework.
We finally discuss the types of PDEs we can describe using this method.

\subsection{Method of Characteristics}
\label{sec:moc}
The MoC provides a procedure for transforming certain PDEs into ODEs along paths known as \emph{characteristics}. 
In the most general sense, the method applies to general hyperbolic differential equations; however, for illustration purposes, we will consider a canonical example using the inviscid Burgers equation.
A complete exposition on the topic can be found in \citet[Chapter 9]{griffiths2015essential}, but we will introduce some basic concepts here for completeness. 
Let $u(x,t) : \mathbb{R} \times \mathbb{R}_+ \to \mathbb{R}$ satisfy the following inviscid Burgers equation
\begin{equation}
\label{eq:burger}
\frac{\partial u}{\partial t} + u \frac{\partial u}{\partial x} = 0,
\end{equation}
where we dropped the dependence on $x$ and $t$ for ease of notation. 
We are interested in the solution of $u$ over some bounded domain $\Omega \subset \mathbb{R} \times \mathbb{R}_+$.
We will now introduce parametric forms for the spatial component $x(s) : [0,1] \to \mathbb{R}$ and temporal components $t(s) : [0,1] \to \mathbb{R}_+$ over the fictitious variable $s \in [0,1]$. 
Intuitively, this allows us to solve an equation on curves $x, t$ as functions of $s$ which we denote $(x(s), t(s))$ as the \emph{characteristic}.
Expanding, and writing $\mathrm{d}$ as the total derivative, we get
\begin{equation}
\frac{\mathrm{d}}{\mathrm{d}s} u(x(s), t(s)) = \frac{\partial u}{\partial x}\frac{dx}{ds} + \frac{\partial u}{\partial t}\frac{dt}{ds}.
\label{eq:chara_expansion}
\end{equation}
Recalling the original PDE in \eqref{eq:burger} and substituting the proper terms into \eqref{eq:chara_expansion} for $\mathrm{d}x/\mathrm{d}s = u,\, \mathrm{d}t/\mathrm{d}s = 1, \,\mathrm{d}u/\mathrm{d}s = 0$, we then recover~\eqref{eq:burger}. 
Solving these equations, we can obtain the characteristics as $x(s) = us + x_0$ and $t(s) = s + t_0$ as functions of initial conditions $x_0, t_0$, and in the case of time, we let $t_0 = 0$.
Finally, by solving over many initial conditions $x_0 \in \partial \Omega$, we can obtain the solution of the PDE over $\Omega$. 
Putting it all together, we have a new ODE that is written as 
$$
\frac{\mathrm{d}}{\mathrm{d}s} u(x(s),t(s)) = \frac{\partial u}{\partial t} + u \frac{\partial u}{\partial x} = 0,
$$
\noindent where we can integrate over $s$ through
\begin{align*}
u(x(T),t(T);x_0,t_0) &:= \int_0^T \frac{\mathrm{d}}{\mathrm{d}s}u(x(s), t(s))\mathrm{d} s \\
 &:= \int_0^T \frac{\mathrm{d}}{\mathrm{d}s}u(us+x_0, s)\mathrm{d} s,
\end{align*}
using the adjoint method with boundary conditions $x_0, t_0$.
This contrasts the usual direct integration over $t$ that is done in NODE; we now jointly couple the integration through the characteristics.
An example of solving this equation over multiple initial conditions is given in Figure~\ref{fig:motivation}. 

\subsection{Neural Representation of Characteristics}
\label{sec:neural_chara}
In the proposed method, we learn the components involved in the MoC, namely the characteristics and the function coefficients.
We now generalize the example given in~\ref{sec:moc}, which involved two variables, to a $k$-dimensional system. 
Specifically, consider the following nonhomogeneous boundary value problem (BVP)

\begin{align}
\label{eq:qlpde}
\begin{cases}
    \frac{\partial \mathbf{u}}{\partial t}+\sum_{i=1}^k a_{i}(x_1,...,x_k,\mathbf{u})\frac{\partial \mathbf{u}}{\partial x_{i}}=\mathbf{c}(x_1,...,x_k,\mathbf{u}), & \text{on } \mathbf{x}, t \in \mathbb{R}^k\times[0,\infty) \\
    \mathbf{u}(\mathbf{x}(0)) = \mathbf{u}_0, &\text{on } \mathbf{x} \in \mathbb{R}^k.
\end{cases}
\end{align}

Here, $\mathbf{u}:\mathbb{R}^{k}\rightarrow\mathbb{R}^{n}$ is a multivariate map, $a_{i} : \mathbb{R}^{k+ n} \to \mathbb{R}$ and $\mathbf{c} : \mathbb{R}^{k+ n}\to \mathbb{R}^{n}$ be functions dependent on values of $\mathbf{u}$ and $x$'s. This problem is well-defined and has a solution so long as $\sum_{i=1}^k a_{i}\frac{\partial \mathbf{u}}{\partial x_{i}}$ is continuous~\citep{evans2010pde}.

MoC has historically been used in a scalar context, but generalization to the vector case is relatively straightforward. 
A proof of the generalization can be found in Appendix~\ref{proof:moc}. 
Following MoC, we decompose the PDE in~\eqref{eq:qlpde} into the following system of ODEs
\begin{align}
    \frac{dx_{i}}{ds}&=a_{i}(x_1,...,x_k,\mathbf{u}),\\
    \frac{d \mathbf{u}}{ds}&=\sum_{i=1}^{k}\frac{\partial \mathbf{u}}{\partial x_{i}}\frac{dx_{i}}{ds}=\mathbf{c}(x_1,...,x_k,\mathbf{u}).
    \label{eq:odesys}
\end{align}
We represent this ODE system by parameterizing $d x_{i}/ds$ and $\partial \mathbf{u}/\partial x_{i}$ with neural networks. 
Consequently, $d\mathbf{u}/ds$ is evolving according to~\eqref{eq:odesys}.


Following this expansion, we arrive at
\begin{align}
\label{eq:integral}
\mathbf{u}(\mathbf{x}(T)) & = \mathbf{u}(\mathbf{x}(0)) + \int_0^{T} \frac{\mathrm{d}\mathbf{u}}{\mathrm{d}s}\left(\mathbf{x},\mathbf{u}\right)\mathrm{d}s  \\
& \nonumber =\mathbf{u}(\mathbf{x}(0)) + \int_0^{T} [\mathbf{J}_\mathbf{x} \mathbf{u}]\left (\mathbf{x} , \mathbf{u}; \Theta_2 \right) \frac{d \mathbf{x}}{d s} \left ( \:\mathbf{x},\mathbf{u};  \Theta_2 \right ) \mathrm{d}s,
\end{align}
where we remove $\mathbf{u}$'s dependency on $\mathbf{x}(s)$ and $\mathbf{x}$'s dependency on $s$ for simplicity of notation. 
In Equation \eqref{eq:integral}, the functions $\mathbf{J}_\mathbf{x}\mathbf{u}$ and $d\mathbf{x}/ds$ are learnable functions which are the outputs of deep neural networks with inputs $\mathbf{x},\,\mathbf{u}$ and parameters $\Theta_2$.

\subsection{Conditioning on data}
Previous works primarily modeled the task of classifying a set of data points with a fixed differential equation, neglecting possible structural variations lying in the data. Here, we condition C-NODE on each data point, thereby solving a PDE with a different initial condition. 
Specifically, consider the term given by the integrand in~\eqref{eq:integral}.
The neural network representing the characteristic $d \mathbf{x}/ds$ is conditioned on the input data $\mathbf{z}\in \mathbb{R}^{w}$.
Define a feature extractor function $\mathbf{g}(\cdot):\, \mathbb{R}^{w}\to \mathbb{R}^{n}$ and we have 
\begin{align}
\label{eq:condition}
    \frac{dx_{i}}{ds}=a_{i}(x_1,\ldots,x_k, \mathbf{u};\mathbf{g}(\mathbf{z})).
\end{align}

By introducing $\mathbf{g}(\mathbf{z})$ in~\eqref{eq:condition}, the equation describing the characteristics changes depending on the current data point. 
This leads to the classification task being modeled with a family rather than one single differential equation. 
\subsection{Training C-NODEs}
After introducing the main components of C-NODEs, we can integrate them into a unified algorithm.
To motivate this section, and to be consistent with part of the empirical evaluation, we will consider classification tasks with data $\left \{ \left (\mathbf{z}_j, \mathbf{y}_j\right)\right\}_{j=1}^N, \:\: \mathbf{z}_j \in \mathbb{R}^w, \:\: \mathbf{y}_j \in \mathbb{Z}^+$.
For instance, $\mathbf{z}_j$ may be an image, and $\mathbf{y}_j$ is its class label. 
In the approach we pursue here, the image $\mathbf{z}_j$ is first passed through a feature extractor function $\mathbf{g}(\cdot;\Theta_1):\mathbb{R}^w\rightarrow\mathbb{R}^n$ with parameters $\Theta_1$. The output of $\mathbf{g}$ is the feature $\mathbf{u}_0^{(j)} =  \mathbf{g}(\mathbf{z}_j;\Theta_1)$ that provides the boundary condition for the PDE on $\mathbf{u}^{(j)}$. 
We integrate along different characteristic curves indexed by $s \in [0,T]$ with boundary condition  $\mathbf{u}^{(j)}(\mathbf{x}(0)) = \mathbf{u}_0^{(j)}$, and compute the end values as given by~\eqref{eq:integral}, where we mentioned in Section~\ref{sec:neural_chara}, 
\begin{align}
 \mathbf{u}^{(j)}(\mathbf{x}(T))
&= \mathbf{u}_0^{(j)} + \int_0^{T} \mathbf{J}_\mathbf{x} \mathbf{u}^{(i)}\left (\mathbf{x} ,\mathbf{u}^{(j)}; \Theta_2 \right)  \frac{d \mathbf{x}}{d s} \left (\mathbf{x},\mathbf{u}^{(j)}; \mathbf{u}_0^{(j)}; \Theta_2 \right ) \mathrm{d}s 
\label{eq:full_writeup}
\end{align}
Finally, $\mathbf{u}^{(j)}(\mathbf{x}(T))$ is passed through another neural network, $\Phi(\mathbf{u}^{(j)}(\mathbf{x}(T));\Theta_3)$ with input $\mathbf{u}^{(j)}(\mathbf{x}(T))$ and parameters $\Theta_3$ whose output are the probabilities of each class labels for image $\mathbf{z}_j$. The entire learning is now is reduced to finding optimal weights $(\Theta_1, \Theta_2, \Theta_3)$ which can be achieved by minimizing the loss 
\begin{align*}
    \mathcal{L}=\sum_{j=1}^N L(\Phi(\mathbf{u}^{(j)}(\mathbf{x}(T));\Theta_3),\mathbf{y}_j),
\end{align*}

where $L(\cdot)$ is a loss function of choice. 
In Algorithm~\ref{alg:training}, we illustrate the implementation procedure with the forward Euler method for simplicity for the framework but note any ODE solver can be used. 

\subsection{Combining MoC with Existing NODE Modifications}
As mentioned in the Section~\ref{sec:related}, the proposed C-NODEs method can be used as an extension to existing NODE frameworks. 
In all NODE modifications, the underlying expression of $\int_a^b \mathbf{f}(t, \mathbf{u}; \Theta) \mathrm{d} t$ remains the same.
Modifying this expression to $\int_a^b \mathbf{J}_{\mathbf{x}} \mathbf{u}(\mathbf{x},\mathbf{u}; \Theta) d \mathbf{x}/ds(\mathbf{x},\mathbf{u};\mathbf{u}_0; \Theta)  \mathrm{d} s$ results in the proposed C-NODE architecture, with the size of $\mathbf{x}$ being a hyperparameter.

\begin{algorithm}[tb]
    \caption{C-NODE algorithm using the forward Euler method}
    \label{alg:training}
\begin{algorithmic}
    \FOR{each input data $\mathbf{z}_j$}
    \STATE extract image feature $\mathbf{u}(s=0)=\mathbf{g}(\mathbf{z}_j;\Theta_1)$ with a feature extractor neural network.
    \STATE {\bfseries procedure} Integration along $s=0\rightarrow 1$
    \FOR {each time step $s_m$}
    \STATE calculate $\frac{d \mathbf{x}}{d s}(\mathbf{x},\mathbf{u};\mathbf{g}(\mathbf{z}_j;\Theta_1);\Theta_2)$ and $\mathbf{J}_\mathbf{x} \mathbf{u}(\mathbf{x},\mathbf{u};\Theta_2)$.
    \STATE calculate $\frac{d\mathbf{u}}{ds}=\mathbf{J}_\mathbf{x} \mathbf{u} \,\frac{d \mathbf{x}}{d s}$.
    \STATE calculate $\mathbf{u}(s_{m+1})=\mathbf{u}(s_m)+\frac{d\mathbf{u}}{ds}(s_{m+1}-s_m)$.
    \ENDFOR
    \STATE {\bfseries end procedure}
    \STATE classify $\mathbf{u}(s=1)$ with neural network $\Phi(\mathbf{u}(\mathbf{x}(s=1)),\Theta_3)$.
    \ENDFOR
\end{algorithmic}
\end{algorithm}

\section{Properties of C-NODEs}
C-NODE has a number of theoretical properties that contribute to its expressiveness.
We provide some theoretical results on these properties in the proceeding sections.
We also define continuous normalizing flows (CNFs) with C-NODEs, extending the CNFs originally defined with NODEs.
\subsection{Intersecting trajectories}

As mentioned in~\citep{dupont2019augmented}, one limitation of NODE is that the mappings cannot represent intersecting dynamics. 
We prove by construction that the C-NODEs can represent some dynamical systems with intersecting trajectories in the following proposition: 

\begin{proposition}
\label{prop:intersect}
The C-NODE can represent a dynamical system on $u(s)$, $du/ds=\mathcal{G}(s,u): \mathbb{R}_+ \times \mathbb{R} \rightarrow \mathbb{R}$, where when $u(0)=1$, then $u(1 )=u(0)+\int_0^1\mathcal{G}(s,u)ds=0$; and when $u(0)=0$, then $u(1)=u(0)+\int_0^1\mathcal{G}(s,u)ds=1$. 
\end{proposition}
\begin{proof}
See Appendix~\ref{proof:intersect}.
\end{proof}

\subsection{Density estimation with C-NODEs}
C-NODEs can also be used to define a continuous density flow that models the density of a variable over space subject to the variable satisfying a PDE. 
Similar to the change of log probability of NODEs, as in~\eqref{eqn:odeflow}, we provide the following proposition for C-NODEs:
\begin{proposition}
\label{prop:cnodelogprob}
Let $u(s)$ be a finite continuous random variable with probability density function $p(u(s))$ and let $u(s)$ satisfy $\frac{du(s)}{ds}=\sum_{i=1}^k\frac{\partial u}{\partial x_i}\frac{dx_i}{ds}$. Assuming $\frac{\partial u}{\partial x_i}$ and $\frac{dx_i}{ds}$ are uniformly Lipschitz continuous in $u$ and continuous in $s$, then the evolution of the log probability of $u$ follows:
\begin{align*}
    \frac{\partial \log p(u(s))}{\partial s}=-\mathrm{tr}\left(\frac{\partial }{\partial u}\sum_{i=1}^k\frac{\partial u}{\partial x_i}\frac{dx_i}{ds}\right)
\end{align*}
\end{proposition}
\begin{proof}
See Appendix~\ref{proof:cnodelogprob}.
\end{proof}
CNFs are continuous and invertible one-to-one mappings onto themselves, i.e., homeomorphisms. \citet{Zhang2020ApproximationCO} proved that vanilla NODEs are not universal estimators of homeomorphisms, and augmented neural ODEs (ANODEs) are universal estimators of homeomorphisms. We demonstrate that C-NODEs are universal estimators of homeomorphisms, which we formalize in the following proposition:
\begin{proposition}
\label{prop:chomeo}
Given any homeomorphism $h:\Upsilon\rightarrow\Upsilon$, $\Upsilon\subset\mathbb{R}^p$, initial condition $u_0$, and time $T>0$, there exists a flow $u(s,u_0)\in\mathbb{R}^n$ following $\frac{du}{ds}=\frac{\partial u}{\partial x}\frac{dx}{ds}+\frac{\partial u}{\partial t}\frac{dt}{ds}$ such that $u(T,u_0)=h(u_0)$.
\end{proposition}
\begin{proof}
See Appendix~\ref{proof:chomeo}.
\end{proof}

\section{Experiments}
We present experiments on image classification tasks on benchmark datasets, image generation tasks on benchmark datasets, PDE modeling, and time series prediction.
\subsection{Classification Experiments with Image Datasets}
We first conduct experiments for classification tasks on high-dimensional image datasets, including MNIST, CIFAR-10, and SVHN.
We provide results for C-NODE and also combine the framework with existing methods, including ANODEs~\citep{dupont2019augmented}, Input Layer NODEs (IL-NODEs)~\citep{massaroli2021dissecting}, and 2nd-Order NODEs~\citep{massaroli2021dissecting}. 

The results for the experiments with the adjoint method are reported in Table~\ref{tab:result} and in Figure \ref{fig:training_adjoint}. We investigate the performances of the models on classification accuracy and the number of function evaluations (NFE) taken in the adaptive numerical integration. NFE is a indicator of the model's computational complexity, and can also be interpreted as the network depth for the continuous NODE system \citep{chen2019neural}. Using a similar number of parameters, combining C-NODEs with different models consistently results in higher accuracies and mostly uses smaller numbers of NFEs, indicating a better parameter efficiency. The performance improvements can be observed, especially on CIFAR-10 and SVHN, where it seems the dynamics to be learned are too complex for ODE systems, requiring a sophisticated model and a large number of NFEs. It appears that solving a PDE system along a multidimensional characteristic is beneficial for training more expressive functions with less complex dynamics, as can be seen in Figures \ref{fig:training_adjoint}, \ref{fig:training}.

\begin{figure*}[tb]
    \centering
    \begin{subfigure}[b]{0.49\textwidth}
        \centering
    \includegraphics[width=\textwidth, trim= 4cm 0cm 4cm 3.5cm,clip]{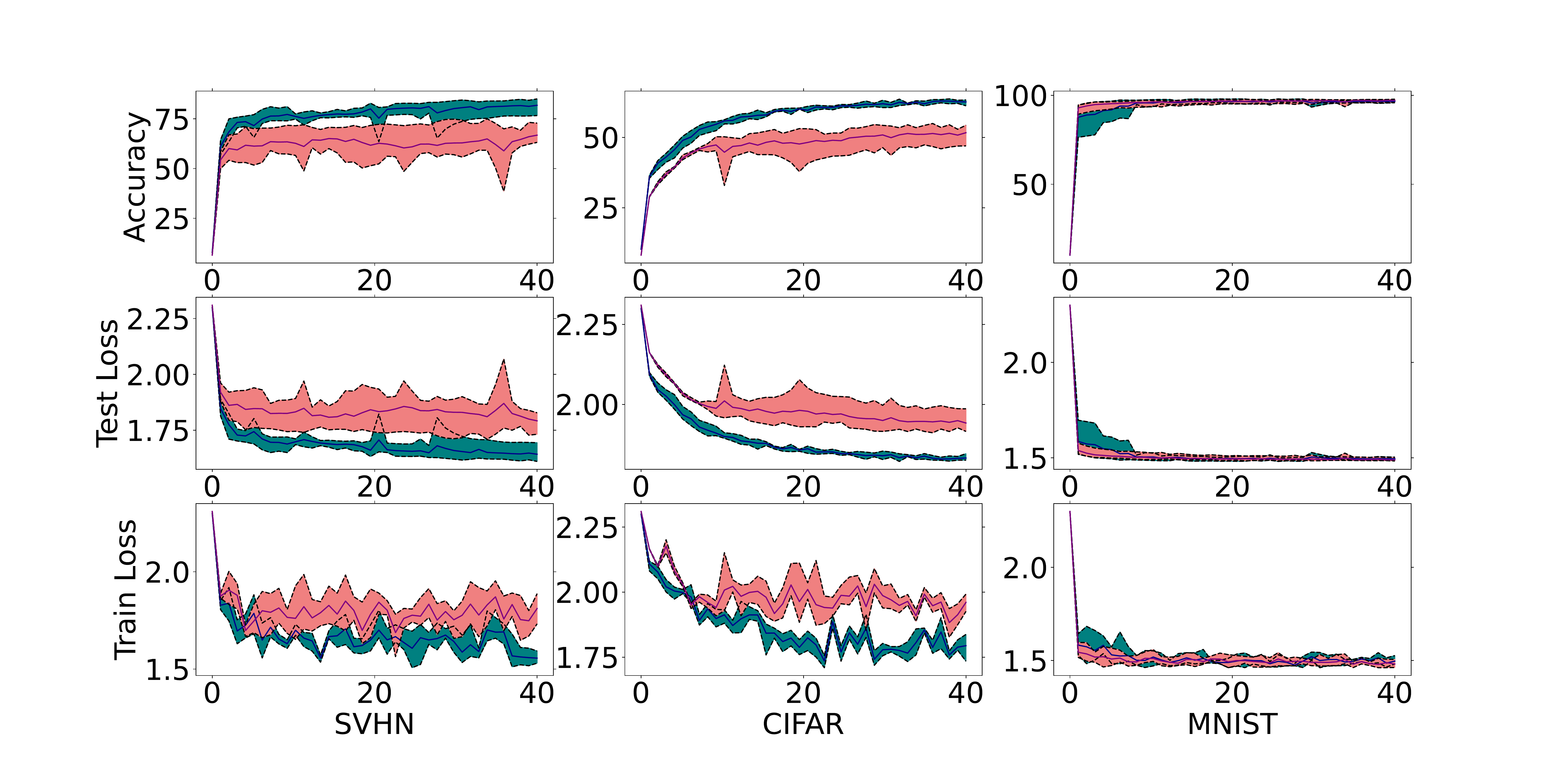}
    \caption{Adjoint training;}
    \label{fig:training_adjoint}
    \end{subfigure} 
    \begin{subfigure}[b]{0.49\textwidth}
        \centering
    \includegraphics[width=\textwidth, trim= 4cm 0cm 4cm 3.5cm,clip]{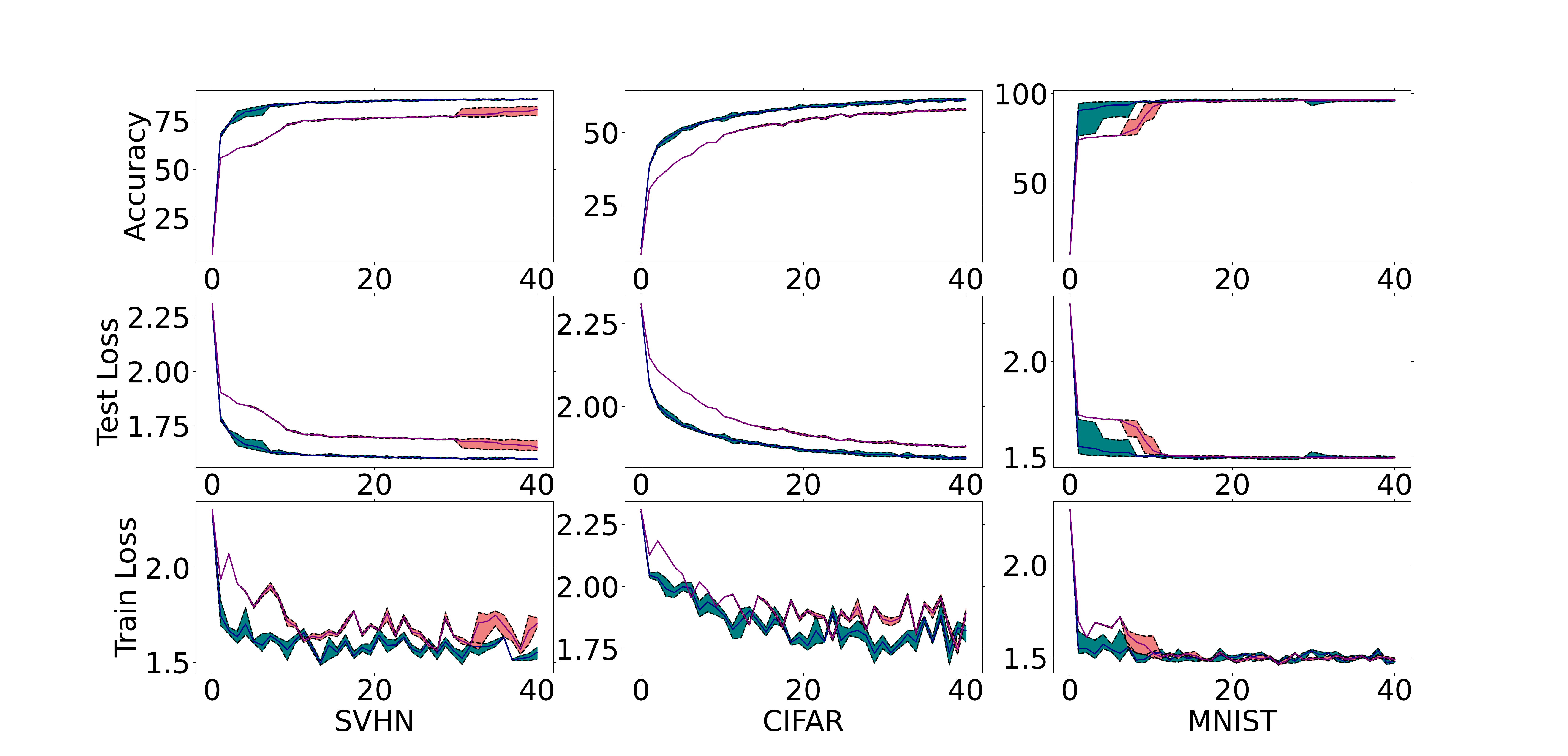}
    \caption{Backprop through Euler training;}
    \label{fig:training}
    \end{subfigure}
    \caption{{\bf Red: NODE. Blue: C-NODE.} Training dynamics of different datasets with adjoint in Fig.~\ref{fig:training_adjoint} and with Euler in Fig.~\ref{fig:training} averaged over five runs. The first column is the training process of SVHN, the second column is of CIFAR-10, and the third column is of MNIST. By incorporating the C-NODE method, we achieve a more stable training process in both CIFAR-10 and SVHN, while achieving higher accuracy. Full-sized figure in supplementary materials.}
    \label{fig:training_dynamics}
\end{figure*}



We also report training results using a traditional backpropagation through the forward Euler solver in Figure~\ref{fig:training}. 
The experiments are performed using the same network architectures as the previous experiments using the adjoint method. 
It appears that C-NODEs converge significantly faster than the NODEs (usually in one epoch) and generally have a more stable training process with smaller variance.
In experiments with MNIST, C-NODEs converge in only one epoch, while NODEs converge in roughly 15 epochs. 
This provides additional empirical evidence on the benefits of training using the characteristics. 
As shown in Figures \ref{fig:training_adjoint}, \ref{fig:training}, compared to training with the adjoint method, training with the forward Euler solver results in less variance, indicating a more stable training process. At the same time, training with the adjoint method results in more accurate models, as the adjoint method uses a constant amount of memory, and can employ more accurate adaptive ODE solvers.

\renewcommand\arraystretch{1.25}
\begin{table*}[ht!]
\centering
\footnotesize
\vspace{0.4cm}
\begin{tabular}{@{}lllll@{}}
\toprule
Dataset     &   Method      & Accuracy $\uparrow$ & NFE $\downarrow$       & Param.[K] $\downarrow$  \\ \hline
\multirow{8}{*}{SVHN} & NODE              & $75.28\pm 0.836\%$               &           131    & 115.444\\
& C-NODE             & $\mathbf{82.19\pm 0.478\%}$      & \bfseries 124    & 113.851   \\ \cline{2-5}
& ANODE             & $89.8\pm0.952\%$                 &           167    & 112.234   \\
& ANODE+C-NODE       & $\mathbf{92.23\pm0.176\%}$      & \bfseries 146    & 112.276   \\ \cline{2-5}
& 2nd-Ord           & $88.22\pm1.11 \%$               &           161    & 112.801   \\
& 2nd-Ord+C-NODE     & $\mathbf{92.37\pm0.118\%}$      & \bfseries 135    & 112.843   \\ \cline{2-5}
& IL-NODE           & $89.69\pm0.369\%$               &           195    & 113.368   \\
& IL-NODE+C-NODE     & $\mathbf{93.31\pm0.088\%}$      & \bfseries 95     & 113.752   \\ \hline
\multirow{8}{*}{ CIFAR-10} & NODE              & $56.30\pm 0.742\%$               &           152    & 115.444 \\
& C-NODE             & $\mathbf{64.28\pm 0.243\%}$     & \bfseries 151    & 113.851 \\ \cline{2-5}
& ANODE             & $70.99\pm 0.483\%$               & \bfseries 177    & 112.234 \\
& ANODE+C-NODE       & $\mathbf{71.36\pm0.220\%}$      &           224    & 112.276 \\ \cline{2-5}
& 2nd-Ord           & $70.84\pm0.360\%$               &           189    & 112.801 \\
& 2nd-Ord+C-NODE     & $\mathbf{73.68\pm0.153\%}$      & \bfseries 131    & 112.843 \\ \cline{2-5}
& IL-NODE           & $72.55\pm0.238\%$               &           134    & 113.368 \\
& IL-NODE+C-NODE     & $\mathbf{73.78\pm0.154\%}$      & \bfseries 85     & 113.752 \\ \hline
\multirow{8}{*}{ MNIST} & NODE              & $96.90\pm 0.154\%$              &           72     & 85.468       \\
& C-NODE             & $\mathbf{97.56\pm 0.431\%}$     &           72     & 83.041        \\ \cline{2-5}
& ANODE             & $99.12\pm 0.021\%$              &           68     & 89.408   \\
& ANODE+C-NODE       & $\mathbf{99.20\pm0.002\%}$      & \bfseries 60     & 88.321   \\ \cline{2-5}
& 2nd-Ord           & $99.35\pm0.002\%$               & \bfseries 52     & 89.552   \\
& 2nd-Ord+C-NODE     & $\mathbf{99.38\pm0.037\%}$      &           61     & 88.465   \\ \cline{2-5}
& IL-NODE           & $99.33\pm0.039\%$               & \bfseries 53     & 89.597   \\
& IL-NODE+C-NODE     & $\mathbf{99.33\pm0.001\%}$               &           60     & 88.51   \\ \bottomrule
\end{tabular}
\caption{Mean test results over 5 runs of different NODE models over SVHN, CIFAR-10, and MNIST. Accuracy and NFE at convergence are reported. Applying C-NODE always increases models' accuracy and usually reduces models' NFE as well as the standard error.}
\label{tab:result}
\end{table*}

\subsection{Ablation study on the dimensionality of C-NODE on classification tasks}
We perform an ablation study on the impact of the number of dimensions of the C-NODE we implement. This study allows us to evaluate the relationship between the model performance and the model's limit of mathematical approximating power. Empirical results show that as we increase the number of dimensions used in the C-NODE model, the C-NODE's performance first improves and then declines, due to overfitting. We have found out that information criteria like AIC and BIC can be successfully applied for dimensionality selection in this scenario. Details of the ablation study can be found in Appendix \ref{sec:ablation}. 

\subsection{Continuous normalizing flow with C-NODEs}
We compare the performance of CNFs defined with NODEs to with C-NODEs on MNIST, SVHN, and CIFAR-10. We use a Hutchinson trace estimator to calculate the trace and use multi-scale convolutional architectures as done in~\citet{Dinh2017DensityEU,grathwohl2019ffjord} \footnote[2]{This is based on the code that the authors of \citet{grathwohl2019ffjord} provided in \url{https://github.com/rtqichen/ffjord}}. 
Differential equations are solved using the adjoint method and a Runge-Kutta of order 5 of the Dormand-Prince-Shampine solver. 
Although the Euler forward method is faster, experimental results show that its fixed step size often leads to negative Bits/Dim, indicating the importance of adaptive solvers. 
As shown in table \ref{tab:cflow} and figure \ref{fig:mnist_cnf}, using a similar number of parameters, experimental results show that CNFs defined with C-NODEs perform better than CNFs defined with NODEs in terms of Bits/Dim, as well as having lower variance, and using a lower NFEs on all of MNIST, CIFAR-10, and SVHN.




\begin{figure}[hbt!]
    \centering
    \includegraphics[width=0.6\textwidth, trim= 0cm 0cm 0cm 4.5cm]{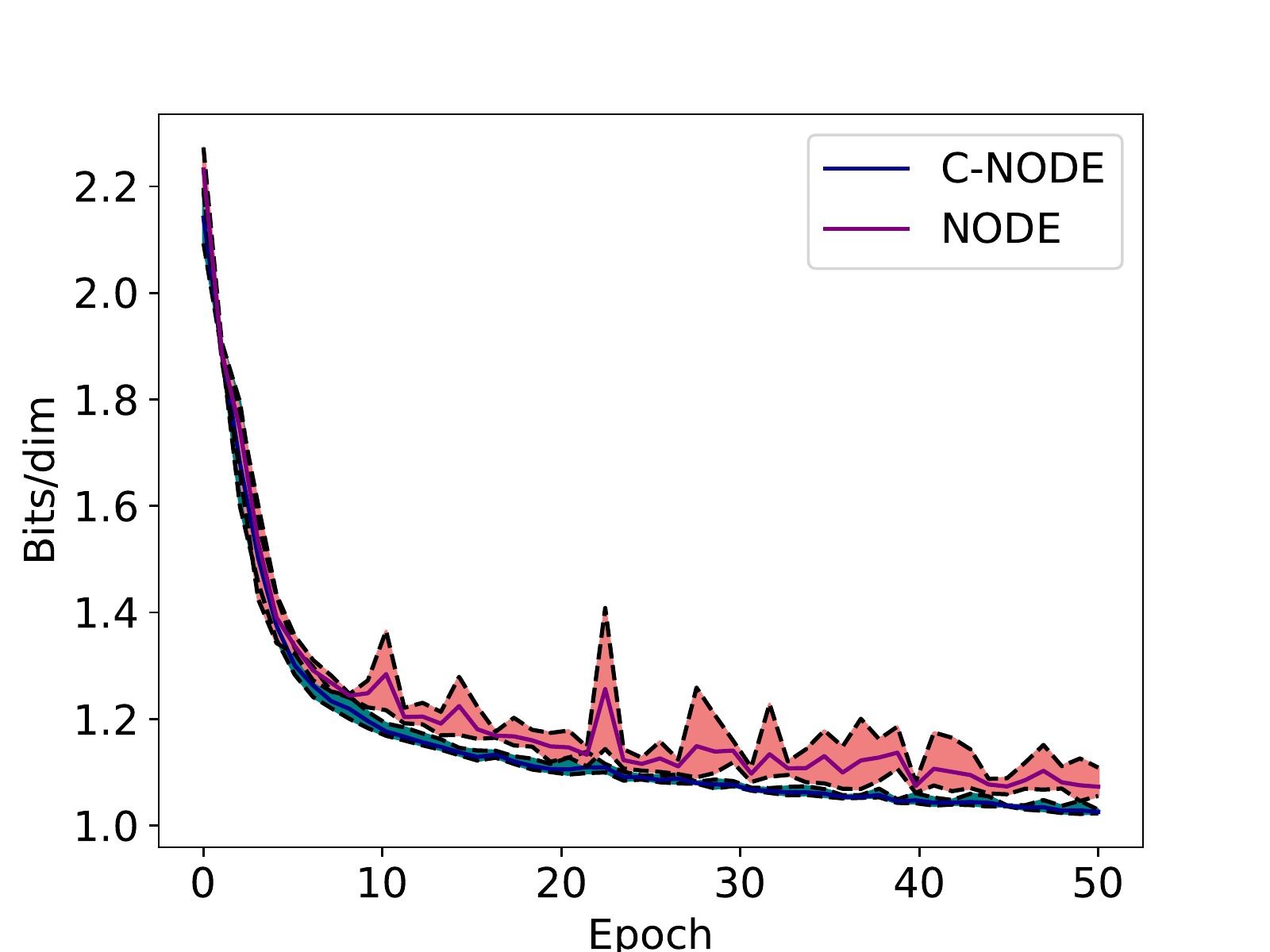}
    \caption{\textbf{Red: NODE. Blue: C-NODE.} Training dynamics of CNFs on MNIST dataset with adjoint method. We present Bits/dim of the first 50 training epochs. C-NODE method achieves higher accuracy, while having a lower variance.}
    \label{fig:mnist_cnf}
\end{figure}

\begin{table}
\centering
\begin{tabular}{l|lll|lll|lll}
\toprule
                                      & \multicolumn{3}{l|}{\textbf{MNIST}}          & \multicolumn{3}{l|}{\textbf{CIFAR-10}}                                                    & \multicolumn{3}{l}{\textbf{SVHN}} \\
\multirow{-2}{*}{Model}               & B/D    & Param.                   & NFE    & B/D                           & Param.                    & NFE                       & B/D       & Param.   & NFE      \\ \hline
Real NVP \citep{Dinh2017DensityEU}   & $1.05$ & {\color[HTML]{9B9B9B} N/A} & --     & {\color[HTML]{333333} $3.49$} & {\color[HTML]{9B9B9B} N/A} & {\color[HTML]{000000} --} & --        & --         & --       \\
Glow \citep{kingma2018glow}          & $1.06$ & {\color[HTML]{9B9B9B} N/A} & --     & $3.35$                        & $44.0$M                     & --                        & --        & --         & --       \\
RQ-NSF \citep{durkan2019neural}      & --     & --                         & --     & $3.38$                        & $11.8$M                     & --                        & --        & --         & --       \\
Res. Flow \citep{chen2019ResidualFF} & $0.97$ & $16.6$M                     & --     & $\mathbf{3.28}$                        & $25.2$M                     & --                        & --        & --         & --       \\
CP-Flow \citep{Huang2021ConvexPF}    & $1.02$ & $2.9$M                      & --     & $3.40$                        & $1.9$M                      & --                        & --        & --         & --       \\ \hline
NODE                                  & $1.00$ & $\mathbf{336.1}$K                      & $1350$ & $3.49$                        & $410.1$K                      & $1847$                    & $2.15$    & $410.1$K      & $1844$   \\
C-NODE                                & $\mathbf{0.95}$ & $338.0$K                      & $\mathbf{1323}$ & $3.44$                        & $\mathbf{406.0}$K                      & $\mathbf{1538}$                    & $\mathbf{2.12}$    & $\mathbf{406.0}$K      & $\mathbf{1352}$   \\ \bottomrule
\end{tabular}
\caption{Experimental results on generation tasks, with NODE, C-NODE, and other models. B/D indicates Bits/dim. Using a similar amount of parameters, C-NODE outperforms NODE on all three datasets, and have a significantly lower NFE when training for CIFAR-10 and SVHN.}
\label{tab:cflow}
\end{table}

\subsection{PDE modeling with C-NODEs}
We consider a synthetic regression example for a hyperbolic PDE with a known solution.
Since NODEs assume that the latent state is only dependent on a scalar (namely time), they cannot model dependencies that vary over multiple spatial variables required by most PDEs.
We quantify the differences in the representation capabilities by examining how well each method can represent a linear hyperbolic PDE. 
We also modify the assumptions used in the classification and density estimation experiments where the boundary conditions were constant as in~\eqref{eq:qlpde}.
We approximate the following BVP:
\begin{align}
\label{eq:bvp_synthetic}
    \begin{cases}
        u\frac{\partial u}{\partial x}+\frac{\partial u}{\partial t}=u,\\
        u(x,0)=2t,& 1\leq x\leq 2.
    \end{cases}
\end{align}
\eqref{eq:bvp_synthetic} has an analytical solution given by $u(x,t) = \frac{2x \exp(t)}{2 \exp(t) +1}$.
We generate a training dataset by randomly sampling 200 points $(x,t)$, $x\in[1,2]$, $t\in [0,1]$, as well as values $u(x,t)$ at those points. 
We test C-NODE and NODE on 200 points randomly sampled as $(x,t) \in [1,2] \times [0,1]$. 
For this experiment, C-NODE uses 809 parameters while NODE uses 1185 parameters. C-NODE deviates 8.05\% from the test dataset, while NODE deviates 30.52\%. 
Further experimental details can be found in Appendix~\ref{exp:pde}.

\subsection{Time series prediction with C-NODEs}
Finally, we test C-NODEs and NODEs on a synthetic time series prediction problem. 
We define a function by $u(x,t)=\frac{2x \exp(t)}{2\exp(t)+1}$, and we sample $\tilde{u} = u(x,t) + 0.1\epsilon_t$, where $\epsilon_t \sim \mathcal{N}(0, 1)$ over $x\in [1,2]$, $t\in[0,1]$ to generate the training dataset. 
We test the performance on $t \in [n,n+1]$ with $n\in\{2,\ldots,5\}$. 
To make the problem more challenging, $x$ values are omitted, and only $t$ values are provided during both training and testing. 
As shown in Table~\ref{tab:time_series}, C-NODE produces more profound improvements over NODEs as time increases.
\begin{table*}[hbt!]
\centering
\begin{tabular}{@{}lllllll@{}}
\toprule
Time   & {[}0,1{]} & {[}1,2{]} & {[}2,3{]} & {[}3,4{]} & {[}4,5{]} & {[}5,6{]} \\ \midrule
NODE   & $20.63\%$   & $25.00\%$   & $32.40\%$   & $45.91\%$   & $52.25\%$   & $70.01\%$   \\
C-NODE & $\mathbf{20.49\%}$   & $\mathbf{22.18\%}$   & $\mathbf{25.54\%}$   & $\mathbf{21.96\%}$   & $\mathbf{23.88\%}$   & $\mathbf{38.24\%}$   \\ \bottomrule
\end{tabular}
\caption{Time series prediction results for C-NODE and NODE. Errors are percentages of deviation from ground truth. As time goes, C-NODE outperforms NODE more.}
\label{tab:time_series}
\end{table*}
\vspace{-10pt}
\section{Discussion}
We describe an approach for extending NODEs to the case of PDEs by solving a series of ODEs along the characteristics of a PDE. 
The approach applies to any black-box ODE solver and can combine with existing NODE-based frameworks.
We empirically showcase its efficacy on classification tasks while also demonstrating its success in improving convergence using Euler forward method without the adjoint method. Additionally, C-NODE empirically achieves better performances on density estimation tasks, while being more efficient with the number of parameters and using lower NFEs.  C-NODE's efficiency over physical modeling and time series prediction is also highlighted with additional experiments.

\paragraph{Limitations} 
There are several limitations to the proposed method. 
The MoC only applies to hyperbolic PDEs, and we only consider first-order semi-linear PDEs in this paper. 
This may be a limitation since this is a specific class of PDEs that does not model all data. 
We additionally noted that, compared to ANODE, C-NODE's training is not as stable. 
This can be improved by coupling C-NODEs with ANODEs or other methods.

\section*{Acknowledgments}
This work was supported in part by the Office of Naval Research (ONR) under grant number N00014-21-1-2590.
AH was supported by NSF-GRFP.
\clearpage

\balance
\bibliography{example_paper}

\begin{thebibliography}{53}
\providecommand{\natexlab}[1]{#1}
\providecommand{\url}[1]{\texttt{#1}}
\expandafter\ifx\csname urlstyle\endcsname\relax
  \providecommand{\doi}[1]{doi: #1}\else
  \providecommand{\doi}{doi: \begingroup \urlstyle{rm}\Url}\fi

\bibitem[Abdal et~al.(2021)Abdal, Zhu, Mitra, and Wonka]{abdal2021styleflow}
Rameen Abdal, Peihao Zhu, Niloy~J Mitra, and Peter Wonka.
\newblock Styleflow: Attribute-conditioned exploration of stylegan-generated
  images using conditional continuous normalizing flows.
\newblock \emph{ACM Transactions on Graphics (TOG)}, 40\penalty0 (3):\penalty0
  1--21, 2021.

\bibitem[Behrmann et~al.(2019)Behrmann, Duvenaud, and
  Jacobsen]{Behrmann2019InvertibleRN}
Jens Behrmann, David~Kristjanson Duvenaud, and J{\"o}rn-Henrik Jacobsen.
\newblock Invertible residual networks.
\newblock In \emph{ICML}, 2019.

\bibitem[Brunton et~al.(2016)Brunton, Proctor, and
  Kutz]{brunton2016discovering}
Steven~L Brunton, Joshua~L Proctor, and J~Nathan Kutz.
\newblock Discovering governing equations from data by sparse identification of
  nonlinear dynamical systems.
\newblock \emph{Proceedings of the national academy of sciences}, 113\penalty0
  (15):\penalty0 3932--3937, 2016.

\bibitem[Chen et~al.(2019{\natexlab{a}})Chen, Behrmann, Duvenaud, and
  Jacobsen]{chen2019ResidualFF}
Ricky T.~Q. Chen, Jens Behrmann, David~Kristjanson Duvenaud, and
  J{\"o}rn-Henrik Jacobsen.
\newblock Residual flows for invertible generative modeling.
\newblock \emph{arXiv preprint arXiv:1906.02735}, 2019{\natexlab{a}}.

\bibitem[Chen et~al.(2019{\natexlab{b}})Chen, Rubanova, Bettencourt, and
  Duvenaud]{chen2019neural}
Ricky T.~Q. Chen, Yulia Rubanova, Jesse Bettencourt, and David Duvenaud.
\newblock Neural ordinary differential equations.
\newblock \emph{arXiv preprint arXiv:1806.07366}, 2019{\natexlab{b}}.

\bibitem[Chen et~al.(2020)Chen, Amos, and Nickel]{chen2020neural}
Ricky~TQ Chen, Brandon Amos, and Maximilian Nickel.
\newblock Neural spatio-temporal point processes.
\newblock \emph{arXiv preprint arXiv:2011.04583}, 2020.

\bibitem[Cranmer et~al.(2020)Cranmer, Greydanus, Hoyer, Battaglia, Spergel, and
  Ho]{cranmer2020lagrangian}
Miles Cranmer, Sam Greydanus, Stephan Hoyer, Peter Battaglia, David Spergel,
  and Shirley Ho.
\newblock Lagrangian neural networks.
\newblock \emph{arXiv preprint arXiv:2003.04630}, 2020.

\bibitem[De~Brouwer et~al.(2019)De~Brouwer, Simm, Arany, and Moreau]{de2019gru}
Edward De~Brouwer, Jaak Simm, Adam Arany, and Yves Moreau.
\newblock Gru-ode-bayes: Continuous modeling of sporadically-observed time
  series.
\newblock \emph{Advances in neural information processing systems}, 32, 2019.

\bibitem[Dinh et~al.(2017)Dinh, Sohl-Dickstein, and Bengio]{Dinh2017DensityEU}
Laurent Dinh, Jascha Sohl-Dickstein, and Samy Bengio.
\newblock Density estimation using real nvp.
\newblock \emph{arXiv preprint arXiv:1605.08803}, 2017.

\bibitem[Dupont et~al.(2019)Dupont, Doucet, and Teh]{dupont2019augmented}
Emilien Dupont, Arnaud Doucet, and Yee~Whye Teh.
\newblock Augmented neural {ODE}s.
\newblock \emph{arXiv preprint arXiv:1904.01681}, 2019.

\bibitem[Durkan et~al.(2019)Durkan, Bekasov, Murray, and
  Papamakarios]{durkan2019neural}
Conor Durkan, Artur Bekasov, Iain Murray, and George Papamakarios.
\newblock Neural spline flows.
\newblock \emph{Advances in neural information processing systems}, 32, 2019.

\bibitem[Evans(2010)]{evans2010pde}
Lawrence Evans.
\newblock \emph{Partial Differential Equations}.
\newblock American Mathematical Society, 2010.

\bibitem[Grathwohl et~al.(2019)Grathwohl, Chen, Bettencourt, Sutskever, and
  Duvenaud]{grathwohl2019ffjord}
Will Grathwohl, Ricky T.~Q. Chen, Jesse Bettencourt, Ilya Sutskever, and
  David~Kristjanson Duvenaud.
\newblock Ffjord: Free-form continuous dynamics for scalable reversible
  generative models.
\newblock \emph{arXiv preprint arXiv:1810.01367}, 2019.

\bibitem[Greydanus et~al.(2019)Greydanus, Dzamba, and
  Yosinski]{greydanus2019hamiltonian}
Samuel Greydanus, Misko Dzamba, and Jason Yosinski.
\newblock Hamiltonian neural networks.
\newblock \emph{Advances in Neural Information Processing Systems}, 32, 2019.

\bibitem[Griffiths et~al.(2015)Griffiths, Dold, and
  Silvester]{griffiths2015essential}
David~F Griffiths, John~W Dold, and David~J Silvester.
\newblock \emph{Essential partial differential equations}.
\newblock Springer, 2015.

\bibitem[Güler et~al.(2019)Güler, Laignelet, and Parpas]{guler2019robust}
Batuhan Güler, Alexis Laignelet, and Panos Parpas.
\newblock Towards robust and stable deep learning algorithms for forward
  backward stochastic differential equations.
\newblock \emph{arXiv preprint arXiv:1910.11623}, 2019.

\bibitem[He et~al.(2015)He, Zhang, Ren, and Sun]{he2015deep}
Kaiming He, Xiangyu Zhang, Shaoqing Ren, and Jian Sun.
\newblock Deep residual learning for image recognition.
\newblock \emph{arXiv preprint arXiv:1512.03385}, 2015.

\bibitem[Hoogeboom et~al.(2019)Hoogeboom, Van Den~Berg, and
  Welling]{hoogeboom2019emerging}
Emiel Hoogeboom, Rianne Van Den~Berg, and Max Welling.
\newblock Emerging convolutions for generative normalizing flows.
\newblock In \emph{International Conference on Machine Learning}, pages
  2771--2780. PMLR, 2019.

\bibitem[Howard(1998)]{howard1998}
Ralph Howard.
\newblock The gronwall inequality, 1998.
\newblock URL \url{http://people.math.sc.edu/howard/ Notes/gronwall.pdf}.

\bibitem[Huang et~al.(2021)Huang, Chen, Tsirigotis, and
  Courville]{Huang2021ConvexPF}
Chin-Wei Huang, Ricky T.~Q. Chen, Christos Tsirigotis, and Aaron~C. Courville.
\newblock Convex potential flows: Universal probability distributions with
  optimal transport and convex optimization.
\newblock \emph{arXiv preprint arXiv:2012.05942}, 2021.

\bibitem[Jia and Benson(2019)]{jia2019neural}
Junteng Jia and Austin~R Benson.
\newblock Neural jump stochastic differential equations.
\newblock \emph{Advances in Neural Information Processing Systems}, 32, 2019.

\bibitem[Jia and Benson(2020)]{jia2020neural}
Junteng Jia and Austin~R. Benson.
\newblock Neural jump stochastic differential equations.
\newblock \emph{arXiv preprint arXiv:1905.10403}, 2020.

\bibitem[Kidger et~al.(2020)Kidger, Morrill, Foster, and
  Lyons]{kidger2020neural}
Patrick Kidger, James Morrill, James Foster, and Terry Lyons.
\newblock Neural controlled differential equations for irregular time series.
\newblock \emph{arXiv preprint arXiv:2005.08926}, 2020.

\bibitem[Kidger et~al.(2021{\natexlab{a}})Kidger, Foster, Li, and
  Lyons]{kidger2021efficient}
Patrick Kidger, James Foster, Xuechen Li, and Terry Lyons.
\newblock Efficient and accurate gradients for neural sdes.
\newblock \emph{arXiv preprint arXiv:2105.13493}, 2021{\natexlab{a}}.

\bibitem[Kidger et~al.(2021{\natexlab{b}})Kidger, Foster, Li, and
  Lyons]{kidger2021neural}
Patrick Kidger, James Foster, Xuechen Li, and Terry~J Lyons.
\newblock Neural sdes as infinite-dimensional gans.
\newblock In \emph{International Conference on Machine Learning}, pages
  5453--5463. PMLR, 2021{\natexlab{b}}.

\bibitem[Kingma and Dhariwal(2018)]{kingma2018glow}
Durk~P Kingma and Prafulla Dhariwal.
\newblock Glow: Generative flow with invertible 1x1 convolutions.
\newblock \emph{Advances in neural information processing systems}, 31, 2018.

\bibitem[Li et~al.(2020)Li, Wong, Chen, and Duvenaud]{li2020scalable}
Xuechen Li, Ting-Kam~Leonard Wong, Ricky T.~Q. Chen, and David Duvenaud.
\newblock Scalable gradients for stochastic differential equations.
\newblock \emph{arXiv preprint arXiv:2001.01328}, 2020.

\bibitem[Liu et~al.(2019)Liu, Xiao, Si, Cao, Kumar, and Hsieh]{liu2019neural}
Xuanqing Liu, Tesi Xiao, Si~Si, Qin Cao, Sanjiv Kumar, and Cho-Jui Hsieh.
\newblock Neural sde: Stabilizing neural ode networks with stochastic noise.
\newblock \emph{arXiv preprint arXiv:1906.02355}, 2019.

\bibitem[Long et~al.(2018)Long, Lu, Ma, and Dong]{long2018pde}
Zichao Long, Yiping Lu, Xianzhong Ma, and Bin Dong.
\newblock Pde-net: Learning pdes from data.
\newblock In \emph{International Conference on Machine Learning}, pages
  3208--3216. PMLR, 2018.

\bibitem[Long et~al.(2019)Long, Lu, and Dong]{long2019pde}
Zichao Long, Yiping Lu, and Bin Dong.
\newblock Pde-net 2.0: Learning pdes from data with a numeric-symbolic hybrid
  deep network.
\newblock \emph{Journal of Computational Physics}, 399:\penalty0 108925, 2019.

\bibitem[Massaroli et~al.(2021)Massaroli, Poli, Park, Yamashita, and
  Asama]{massaroli2021dissecting}
Stefano Massaroli, Michael Poli, Jinkyoo Park, Atsushi Yamashita, and Hajime
  Asama.
\newblock Dissecting neural odes.
\newblock \emph{arXiv preprint arXiv:2002.08071}, 2021.

\bibitem[Morrill et~al.(2021{\natexlab{a}})Morrill, Kidger, Yang, and
  Lyons]{morrill2021neural}
James Morrill, Patrick Kidger, Lingyi Yang, and Terry Lyons.
\newblock Neural controlled differential equations for online prediction tasks,
  2021{\natexlab{a}}.

\bibitem[Morrill et~al.(2021{\natexlab{b}})Morrill, Salvi, Kidger, Foster, and
  Lyons]{morrill2021rough}
James Morrill, Cristopher Salvi, Patrick Kidger, James Foster, and Terry Lyons.
\newblock Neural rough differential equations for long time series.
\newblock \emph{arXiv preprint arXiv:2009.08295}, 2021{\natexlab{b}}.

\bibitem[Norcliffe et~al.(2020)Norcliffe, Bodnar, Day, Simidjievski, and
  Liò]{norcliffe2020second}
Alexander Norcliffe, Cristian Bodnar, Ben Day, Nikola Simidjievski, and Pietro
  Liò.
\newblock On second order behaviour in augmented neural odes.
\newblock \emph{arXiv preprint arXiv:2006.07220}, 2020.

\bibitem[Papamakarios et~al.(2017)Papamakarios, Murray, and
  Pavlakou]{Papamakarios2017MaskedAF}
George Papamakarios, Iain Murray, and Theo Pavlakou.
\newblock Masked autoregressive flow for density estimation.
\newblock \emph{arXiv preprint arXiv:1705.07057}, 2017.

\bibitem[Raissi et~al.(2019)Raissi, Perdikaris, and Karniadakis]{raissi2019686}
M.~Raissi, P.~Perdikaris, and G.E. Karniadakis.
\newblock Physics-informed neural networks: A deep learning framework for
  solving forward and inverse problems involving nonlinear partial differential
  equations.
\newblock \emph{Journal of Computational Physics}, 378:\penalty0 686--707,
  2019.
\newblock ISSN 0021-9991.
\newblock \doi{https://doi.org/10.1016/j.jcp.2018.10.045}.
\newblock URL
  \url{https://www.sciencedirect.com/science/article/pii/S0021999118307125}.

\bibitem[Rezende and Mohamed(2015)]{JimenezRezende2015VariationalIW}
Danilo~Jimenez Rezende and Shakir Mohamed.
\newblock Variational inference with normalizing flows.
\newblock In \emph{ICML}, 2015.

\bibitem[Rubanova et~al.(2019)Rubanova, Chen, and Duvenaud]{rubanova2019latent}
Yulia Rubanova, Ricky~TQ Chen, and David~K Duvenaud.
\newblock Latent ordinary differential equations for irregularly-sampled time
  series.
\newblock \emph{Advances in neural information processing systems}, 32, 2019.

\bibitem[Ruthotto and Haber(2020)]{ruthotto2020deep}
Lars Ruthotto and Eldad Haber.
\newblock Deep neural networks motivated by partial differential equations.
\newblock \emph{Journal of Mathematical Imaging and Vision}, 62\penalty0
  (3):\penalty0 352--364, 2020.

\bibitem[Salvi et~al.(2022)Salvi, Lemercier, and Gerasimovics]{salvi2022spde}
Cristopher Salvi, Maud Lemercier, and Andris Gerasimovics.
\newblock Neural stochastic partial differential equations:
  Resolution-invariant learning of continuous spatiotemporal dynamics.
\newblock \emph{arXiv preprint arXiv:2110.10249}, 2022.

\bibitem[Song et~al.(2021)Song, Sohl-Dickstein, Kingma, Kumar, Ermon, and
  Poole]{Song2021ScoreBasedGM}
Yang Song, Jascha Sohl-Dickstein, Diederik~P. Kingma, Abhishek Kumar, Stefano
  Ermon, and Ben Poole.
\newblock Score-based generative modeling through stochastic differential
  equations.
\newblock \emph{arXiv preprint arXiv:2011.13456}, 2021.

\bibitem[Strichartz(2000)]{strichartz2000analysis}
Robert Strichartz.
\newblock \emph{The way of analysis}.
\newblock Jones and bartlett mathematics, 2000.

\bibitem[Sun et~al.(2019)Sun, Zhang, and Schaeffer]{sun2019neupde}
Yifan Sun, Linan Zhang, and Hayden Schaeffer.
\newblock Neupde: Neural network based ordinary and partial differential
  equations for modeling time-dependent data.
\newblock \emph{arXiv preprint arXiv:1908.03190}, 2019.

\bibitem[Toth et~al.(2019)Toth, Rezende, Jaegle, Racani{\`e}re, Botev, and
  Higgins]{toth2019hamiltonian}
Peter Toth, Danilo~Jimenez Rezende, Andrew Jaegle, S{\'e}bastien Racani{\`e}re,
  Aleksandar Botev, and Irina Higgins.
\newblock Hamiltonian generative networks.
\newblock \emph{arXiv preprint arXiv:1909.13789}, 2019.

\bibitem[Vahdat et~al.(2021)Vahdat, Kreis, and Kautz]{Vahdat2021ScorebasedGM}
Arash Vahdat, Karsten Kreis, and Jan Kautz.
\newblock Score-based generative modeling in latent space.
\newblock \emph{arXiv preprint arXiv:2106.05931}, 2021.

\bibitem[Weinan(2017)]{e2017dynamical}
E.~Weinan.
\newblock A proposal on machine learning via dynamical systems.
\newblock \emph{Communications in Mathematics and Statistics}, 5, 2017.

\bibitem[Xu et~al.(2022)Xu, Chen, Li, and Duvenaud]{xu2022infinitely}
Winnie Xu, Ricky~TQ Chen, Xuechen Li, and David Duvenaud.
\newblock Infinitely deep bayesian neural networks with stochastic differential
  equations.
\newblock In \emph{International Conference on Artificial Intelligence and
  Statistics}, pages 721--738. PMLR, 2022.

\bibitem[Yildiz et~al.(2019)Yildiz, Heinonen, and
  Lahdesmaki]{yildiz2019ode2vae}
Cagatay Yildiz, Markus Heinonen, and Harri Lahdesmaki.
\newblock Ode2vae: Deep generative second order odes with bayesian neural
  networks.
\newblock \emph{Advances in Neural Information Processing Systems}, 32, 2019.

\bibitem[Yildiz et~al.(2021)Yildiz, Heinonen, and
  L{\"a}hdesm{\"a}ki]{pmlr-v139-yildiz21a}
Cagatay Yildiz, Markus Heinonen, and Harri L{\"a}hdesm{\"a}ki.
\newblock Continuous-time model-based reinforcement learning.
\newblock In Marina Meila and Tong Zhang, editors, \emph{Proceedings of the
  38th International Conference on Machine Learning}, volume 139 of
  \emph{Proceedings of Machine Learning Research}, pages 12009--12018. PMLR,
  18--24 Jul 2021.
\newblock URL \url{https://proceedings.mlr.press/v139/yildiz21a.html}.

\bibitem[Zhang et~al.(2020)Zhang, Gao, Unterman, and
  Arodz]{Zhang2020ApproximationCO}
Han Zhang, Xi~Gao, Jacob Unterman, and Tom Arodz.
\newblock Approximation capabilities of neural odes and invertible residual
  networks.
\newblock In \emph{ICML}, 2020.

\bibitem[Zhang et~al.(2018)Zhang, E, and Wang]{zhang2018mongeampere}
Linfeng Zhang, Weinan E, and Lei Wang.
\newblock Monge-amp\`ere flow for generative modeling.
\newblock \emph{arXiv preprint arXiv:1809.10188}, 2018.

\bibitem[Zhong et~al.(2019)Zhong, Dey, and Chakraborty]{zhong2019symplectic}
Yaofeng~Desmond Zhong, Biswadip Dey, and Amit Chakraborty.
\newblock Symplectic ode-net: Learning hamiltonian dynamics with control.
\newblock \emph{arXiv preprint arXiv:1909.12077}, 2019.

\bibitem[Zhu et~al.(2021)Zhu, Guo, and Lin]{zhu2021neural}
Qunxi Zhu, Yao Guo, and Wei Lin.
\newblock Neural delay differential equations.
\newblock In \emph{International Conference on Learning Representations}, 2021.
\newblock URL \url{https://openreview.net/forum?id=Q1jmmQz72M2}.

\end{thebibliography}
\newpage
\appendix
\onecolumn




\section{Experimental Details}
\subsection{Experimental details of classification tasks}
\label{exp:cla}
We report the average performance over five independent training processes, and the models are trained for 100 epochs for all three datasets.

The input for 2nd-Ord, NODE, and C-NODE are the original images. In the IL-NODE, we transform the input to a latent space before the integration by the integral; that is, we raise the $\mathbb{R}^{c\times h\times w}$ dimensional input image into the $\mathbb{R}^{(c+p)\times h\times w}$ dimensional latent feature space\footnote[1]{This is based on the code that the authors of \citet{massaroli2021dissecting} provide in \url{https://github.com/DiffEqML/torchdyn}}.
We decode the result after performing the continuous transformations along characteristics curves, back to the $\mathbb{R}^{c\times h\times w}$ dimensional object space. 
Combining this with the C-NODE can be seen as solving a PDE on the latest features of the images rather than on the images directly. 
We solve first-order PDEs with three variables in CIFAR-10 and SVHN and solve first-order PDEs with two variables in MNIST. The number of parameters of the models is similar by adjusting the number of features used in the networks. We use similar training hyperparameters as \cite{massaroli2021dissecting}.

Unlike ODEs, we take derivatives with respect to different variables in PDEs. For a PDE with $k$ variables, this results in the constraint of the balance equations
\begin{align*}
    \frac{\partial^2 u}{\partial x_i x_j} = \frac{\partial^2 u}{\partial x_j x_i},\; i,\,j\in\{1,2,...,k\},
    i\neq j.
\end{align*}
This can be satisfied by defining the $k$-th derivative with a neural network, and integrate $k-1$ times to get the first order derivatives. Another way of satisfying the balance equation is to drop the dependency on the variables, i.e., $\forall i\in \{1,2,...,k\}$, 
\begin{align*}
    \frac{\partial u}{\partial x_i}=f_i(u;\theta).
\end{align*}
When we drop the dependency, all higher order derivatives are zero, and the balance equations are satisfied.

All experiments were performed on NVIDIA RTX 3090 GPUs on a cloud cluster.
\subsection{Experimental details of continuous normalizing flows}
\label{exp:cnf}
We report the average performance over four independent training processes. As shown in \cref{fig:flow_training}, compared to NODE, using a C-NODE structure improves the stability of training, as well as having a better performance. Specifically, the standard errors for C-NODEs on MNIST, SVHN, and CIFAR-10 are 0.37\%, 0.51\%, and 0.24\% respectively, and for NODEs the standard errors on MNIST, SVHN, and CIFAR-10 are 1.07\%, 0.32\%, and 0.22\% respectively. 

The experiments are developed using code adapted from the code that the authors of \cite{grathwohl2019ffjord} provided in \url{https://github.com/rtqichen/ffjord}. 

All experiments were performed on NVIDIA RTX 3090 GPUs on a cloud cluster.

\begin{figure*}
    \centering
    \includegraphics[width=0.9\textwidth,trim=0cm 0cm 0cm 0cm,clip]{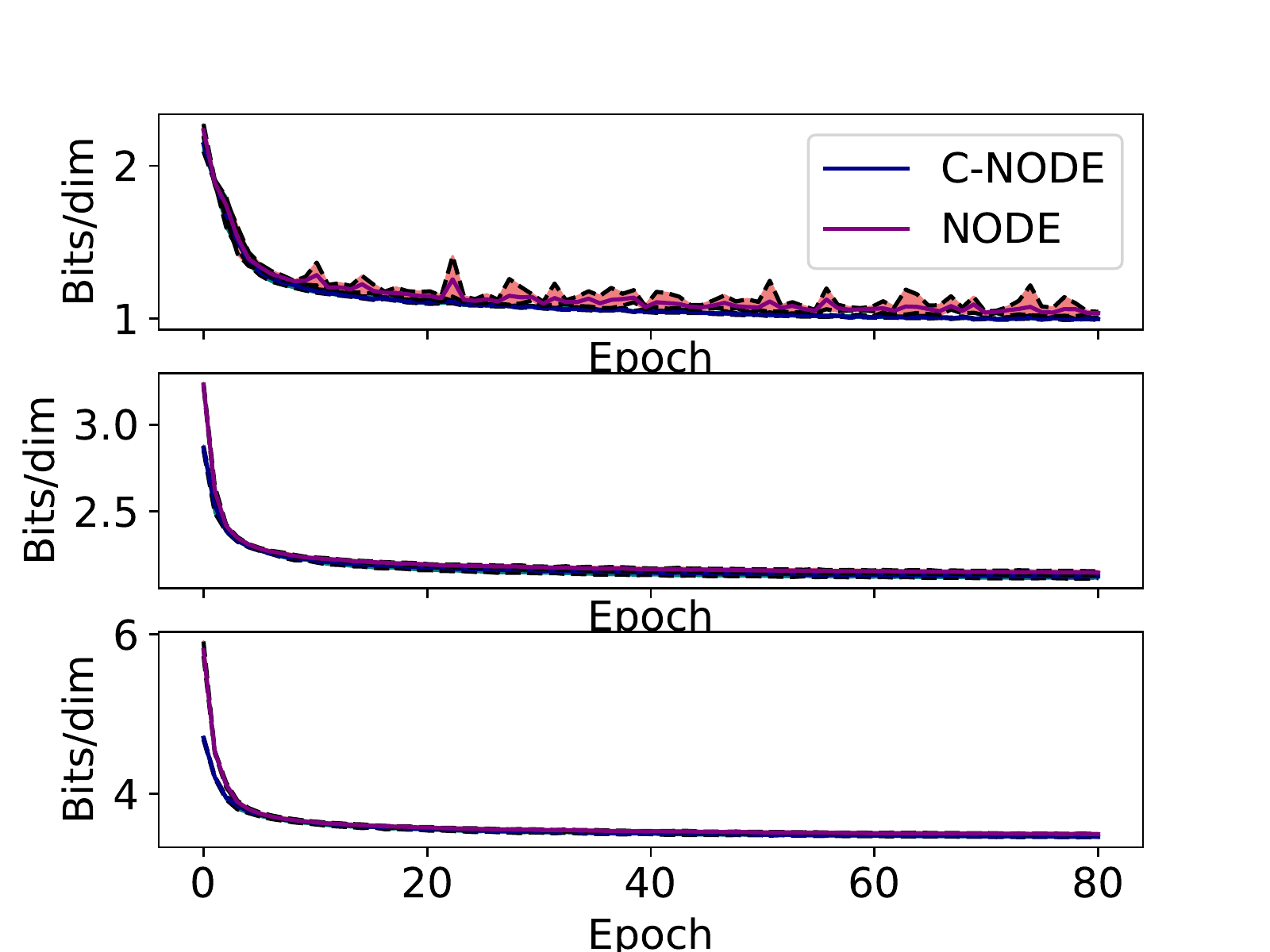}
    \caption{The training process averaged over 4 runs of C-NODE and NODE. The first row are the results on MNIST, the second row are the results on SVHN, the third row are the results on CIFAR-10.}
    \label{fig:flow_training}
\end{figure*}
\subsection{Experimental details of PDE modeling}
\label{exp:pde}
We want to solve the initial value problem
\[\begin{cases}
    u\frac{\partial u}{\partial x}+\frac{\partial u}{\partial t}=u,\\
    u(x,0)=2x, &1\leq x\leq 2,
\end{cases}\]
where the exact solution is $u(x,t)=\frac{2xe^t}{(2e^t+1)}$. Our dataset's input are 200 randomly sampled points $(x,t)$, $x\in[1,2]$, $t\in[0,1]$, and the dataset's outputare the exact solutions at those points.

For the C-NODE architecture, we define four networks: $NN_1(x,t)$ for $\frac{\partial u}{\partial x}$, $NN_2(x,t)$ for $\frac{\partial u}{\partial t}$, $NN_3(t)$ for the characteristic path $(x(s),t(s))$, $NN_4(x)$ for the initial condition. The result is calculated in four steps:
\begin{enumerate}
    \item Integrate $\Delta u = \int_0^t\frac{du(x(s),t(s))}{ds}ds=\int_0^t\frac{\partial u}{\partial t}\frac{dt}{ds}+\frac{\partial u}{\partial x}\frac{dx}{ds}ds=NN_2*NN_3[0]+NN_1*NN_3[1]ds$ as before.
    \item Given $x,\,t$, solve equation $\iota+NN_3(NN_4(\iota))[0]*t=x$ for $\iota$ iteratively, with $\iota_{n+1}=x-NN_3(NN_4(\iota_n))[0]*t$. $\iota_0$ is initialized to be $x$.
    \item Calculate initial value $u(x(0),t(0)) = NN_4(\iota)$.
    \item $u(x,t)=\Delta u+u(x(0),t(0))$.
\end{enumerate}

For the NODE architecture, we define one network: $NN_1(x,t)$ for $\frac{\partial u}{\partial t}$. The result is calculated as $u(x,t)=\int_0^t\frac{\partial u}{\partial t}dt=\int_0^tNN_1dt$.

All experiments were performed on NVIDIA RTX 3080 ti GPUs on a local machine.
\subsection{Experimental details of time series predictions}
\label{exp:time}
We want to predict $u(x,t)=\frac{2\cdot x\cdot e^t}{2\cdot e^t +1}$ at different time $t$, with $x\in[1,2]$, and $x$ being not accessible to the network. We also provide the network with the value of $u(1,0)$.

We use a 8 dimensional C-NODE network. The result is calculated with 
\begin{align*}
    u(x,t) = u(1,0)+\int_0^t\sum_{i=1}^8 \frac{\partial u}{\partial z_i}\frac{dz_i}{ds}ds.
\end{align*}

NODE is calculated with \begin{align*}
    u(x,t)=u(1,0)+\int_0^t\frac{\partial u}{\partial t}dt.
\end{align*} 

C-NODE uses 9744 parameters, and NODE uses 9697 parameters. 

All experiments were performed on NVIDIA RTX 3080 ti GPUs on a local machine.
\section{Approximation Capabilities of C-NODE}

\begin{proposition}[Method of Characteristics for Vector Valued PDEs]
\label{prop:moc}
Let $\mathbf{u}(x_1, \ldots, x_k) : \mathbb{R}^k \to \mathbb{R}^n$ be the solution of a first order semilinear PDE on a bounded domain $\Omega \subset \mathbb{R}^k$ of the form 
\begin{equation}
    \label{eq:semilin-pde}
    \sum_{i=1}^k a_i(x_1,\ldots,x_k,\mathbf{u})\frac{\partial \mathbf{u}}{\partial x_i}=\mathbf{c}(x_1,\ldots, x_k,\mathbf{u}) \quad \text{on} \:\: (x_1, \ldots, x_k) = \mathbf{x} \in \Omega.
\end{equation}
Additionally, let $\mathbf{a} = (a_1, \ldots, a_k)^T : \mathbb{R}^{k + n} \to \mathbb{R}^k, \mathbf{c} : \mathbb{R}^{k + n} \to \mathbb{R}^n$ be Lipschitz continuous functions. 
Define a system of ODEs as
\[
\begin{cases}
    \frac{d \mathbf{x}}{ds}(s) &=\mathbf{a}(\mathbf{x}(s), \mathbf{U}(s))\\
    \frac{d\mathbf{U}}{ds}(s)&=\mathbf{c}(\mathbf{x}(s),\mathbf{U}(s))\\
    \mathbf{x}(0)&\coloneqq \mathbf{x}_0,\,\mathbf{x}_0 \in\partial\Omega\\
    \mathbf{u}(\mathbf{x}_0)&\coloneqq\mathbf{u}_0\\
    \mathbf{U}(0)&\coloneqq\mathbf{u}_0
\end{cases}
\]
where $\mathbf{x}_0$ and $\mathbf{u}_0$ define the initial condition, $\partial \Omega$ is the boundary of the domain $\Omega$. Given initial conditions $\mathbf{x}_0,\mathbf{u}_0$, the solution of this system of ODEs $\mathbf{U}(s) : [a,b] \to \mathbb{R}^d$ is equal to the solution of the PDE in Equation \eqref{eq:semilin-pde} along the characteristic curve defined by $\mathbf{x}(s)$, i.e., $\mathbf{u}(\mathbf{x}(s))=\mathbf{U}(s)$.
The union of solutions $\mathbf{U}(s) \:\:$ for all $\:\: \mathbf{x}_0 \in \partial \Omega$ is equal to the solution of the original PDE in Equation \eqref{eq:semilin-pde} for all $\mathbf{x} \in \Omega$.
\end{proposition}


\begin{lemma}[Gronwall's Lemma \citep{howard1998}]
\label{lemma:gronwall}
Let $U\subset \mathbb{R}^n$ be an open set. Let $\mathbf{f}:U\times [0,T]\rightarrow\mathbb{R}^n$ be a continuous function and let $\mathbf{h_1},\,\mathbf{h_2}:[0,T]\rightarrow U$ satisfy the initial value problems:
$$\frac{d\mathbf{h_1}(t)}{dt}=f(\mathbf{h_1}(t),t),\; \mathbf{h_1}(0)=\mathbf{x_1},$$
$$\frac{d\mathbf{h_2}(t)}{dt}=f(\mathbf{h_2}(t),t),\; \mathbf{h_2}(0)=\mathbf{x_2}.$$
If there exists non-negative constant $C$ such that for all $t\in[0,T]$
$$\|\mathbf{f}(\mathbf{h_2}(t),t)-\mathbf{f}(\mathbf{h_1}(t),t)\|\leq C\|\mathbf{h_2}(t)-\mathbf{h_1}(t)\|,$$
where $\|\cdot\|$ is the Euclidean norm. Then, for all $ t\in [0,T]$, 
$$\|\mathbf{h_2}(t)-\mathbf{h_1}(t)\|\leq e^{Ct}\|\mathbf{x_2}-\mathbf{x_1}\|.$$
\end{lemma}

\subsection{Proof of Proposition \ref{prop:moc}}
\label{proof:moc}
This proof is largely based on the proof for the univarate case provided at\footnote{\url{https://en.wikipedia.org/wiki/Method_of_characteristics##Proof_for_quasilinear_Case}}. 
We extend for the vector valued case. 
\begin{proof}
For PDE on $\mathbf{u}$ with $k$ input, and an $n$-dimensional output, we have $a_i : \mathbb{R}^{k + n} \to \mathbb{R}$, $\frac{\partial \mathbf{u}}{\partial x_i} \in \mathbb{R}^{n}$, and $\mathbf{c} : \mathbb{R}^{k + n} \to \mathbb{R}^{n}$. 
In proposition \ref{prop:moc}, we look at PDEs in the following form 
\begin{equation}
    \sum_{i=1}^k a_i(x_1,\ldots,x_k,\mathbf{u})\frac{\partial \mathbf{u}}{\partial x_i}=\mathbf{c}(x_1,\ldots, x_k,\mathbf{u}).
\end{equation}
Defining and substituting $\mathbf{x} = (x_1, \ldots, x_k)^\intercal$, $\mathbf{a} = (a_1, \ldots, a_k)^\intercal$, and Jacobian $\mathbf{J}(\mathbf{u}(\mathbf{x}))=(\frac{\partial \mathbf{u}}{\partial x_1},...,\frac{\partial \mathbf{u}}{\partial x_k}) \: \in \mathbb{R}^{n\times k}$ into Equation \eqref{eq:semilin-pde} result in  
\begin{equation}
\label{eq:jacob_PDE}
\mathbf{J} ( \mathbf{u} (\mathbf{x})) \mathbf{a}(\mathbf{x},\mathbf{u}) = \mathbf{c}(\mathbf{x}, \mathbf{u}).
\end{equation}
From proposition \ref{prop:moc}, the characteristic curves are given by 
$$
\frac{dx_i}{ds} = a_i(x_1, \ldots, x_k, \mathbf{u}),
$$
and the ODE system is given by
\begin{equation}
\label{eq:chara_ode}
\frac{d\mathbf{x}}{ds}(s)=\mathbf{a}(\mathbf{x}(s),\mathbf{U}(s)),
\end{equation}
\begin{equation}
\label{eq:u_ode}
\frac{d\mathbf{U}}{ds}(s)=\mathbf{c}(\mathbf{x}(s),\mathbf{U}(s)).
\end{equation}
Define the difference between the solution to \eqref{eq:u_ode} and the PDE in \eqref{eq:semilin-pde} as
$$\Delta(s)=\left\|\mathbf{u}(\mathbf{x}(s))-\mathbf{U}(s)\right\|^2=\left(\mathbf{u}(\mathbf{x}(s))-\mathbf{U}(s)\right)^\intercal\left(\mathbf{u}(\mathbf{x}(s))-\mathbf{U}(s)\right),$$ 
Differentiating $\Delta(s)$ with respect to $s$ and plugging in \eqref{eq:chara_ode}, we get
\begin{align}
 \nonumber   \Delta'(s)\coloneqq \frac{d \Delta(s)}{ds} &= 2(\mathbf{u}(\mathbf{x}(s)) -\mathbf{U}(s))\cdot(\mathbf{J}(\mathbf{u}) \mathbf{x}'(s) - \mathbf{U}'(s)) \\
    &= 2[\mathbf{u}(\mathbf{x}(s)) -\mathbf{U}(s)]\cdot[\mathbf{J}(\mathbf{u})\mathbf{a}(\mathbf{x}(s),\mathbf{U}(s)) - \mathbf{c}(\mathbf{x}(s), \mathbf{U}(s))].
    \label{eq:delta_prime}
\end{align}
\eqref{eq:jacob_PDE} gives us $\sum_{i=1}^k a_i(x_1,\ldots,x_k,\mathbf{u})\frac{\partial \mathbf{u}}{\partial x_i}-\mathbf{c}(x_1,\ldots, x_k,\mathbf{u})=0$. 
Plugging this equality into \eqref{eq:delta_prime} and rearrange terms, we have
\begin{align*}
    \Delta'(s) = 2[\mathbf{u}(\mathbf{x}(s)) -\mathbf{U}(s)] &\cdot\{[\mathbf{J}(\mathbf{u})\mathbf{a}(\mathbf{x}(s),\mathbf{U}(s)) - \mathbf{c}(\mathbf{x}(s), \mathbf{U}(s))]\\
    &-  [\mathbf{J}(\mathbf{u})\mathbf{a}(\mathbf{x}(s),\mathbf{u}(s)) - \mathbf{c}(\mathbf{x}(s), \mathbf{u}(s))]\}.
\end{align*}

Combining terms, we have
\begin{align*}
    \Delta'&=2(\mathbf{u}-\mathbf{U})\cdot\left ( \left [\mathbf{J}(\mathbf{u}) \mathbf{a}(\mathbf{U})-\mathbf{c}(\mathbf{U}) \right]-\left [\mathbf{J}(\mathbf{u}) \mathbf{a}(\mathbf{u})-\mathbf{c}(\mathbf{u})\right]\right) \\
    &=2(\mathbf{u}-\mathbf{U})\cdot\left (\mathbf{J}(\mathbf{u})\left[\mathbf{a}(\mathbf{U})-\mathbf{a}(\mathbf{u})\right]+\left[\mathbf{c}(\mathbf{U})-\mathbf{c}(\mathbf{u})\right]\right).
\end{align*}
Applying triangle inequality, we have
\begin{align*}
    \| \Delta' \| \leq 2\|\mathbf{u} - \mathbf{U} \|  ( \| \mathbf{J}(\mathbf{u})  \| \| \mathbf{a}(\mathbf{U}) - \mathbf{a}(\mathbf{u}) \|   + \|\mathbf{c}(\mathbf{U}) - \mathbf{c}(\mathbf{u}) \|  ).
\end{align*}
By the assumption in proposition \ref{prop:moc}, $\mathbf{a}$ and $\mathbf{c}$ are Lipschitz continuous. 
By Lipschitz continuity, we have $\|\mathbf{a}(\mathbf{U})-\mathbf{a}(\mathbf{u}))\|\leq A\|\mathbf{u}-\mathbf{U}\|$ and $\|\mathbf{c}(\mathbf{U})-\mathbf{c}(\mathbf{u}))\|\leq B\|\mathbf{u}-\mathbf{U}\|$, for some constants A and B in $\mathbb{R}_+$. Also, for compact set $[0,s_0]$, $s_0<\infty$, since both $\mathbf{u}$ and Jacobian $\mathbf{J}$ are continuous mapping, $\mathbf{J}(\mathbf{u})$ is also compact. Since a subspace of $\mathbb{R}^n$ is compact if and only it is closed and bounded, $\mathbf{J}(\mathbf{u})$ is bounded \citep{strichartz2000analysis}. Thus, $\|\mathbf{J}(\mathbf{u})\|\leq M$ for some constant $M$ in $\mathbb{R}_+$. 
Define $C=2(A M+B)$, we have
\begin{align*}
    \|\Delta'(s)\|&\leq 2(A M\|\mathbf{u}-\mathbf{U}\|+B\|\mathbf{u}-\mathbf{U}\|)\|\mathbf{u}-\mathbf{U}\| \\
    &=C\|\mathbf{u}-\mathbf{U}\|^2 \\
    &=C\|\Delta(s)\|.
\end{align*}
From proposition \ref{prop:moc}, we have $\mathbf{u}(\mathbf{x}(0))=\mathbf{U}(0)$. As proved above, we have 
$$\left \|\frac{d\mathbf{u}(\mathbf{x}(s))}{ds}-\frac{d\mathbf{U}(s)}{ds}\right\|\coloneqq\|\Delta'(s)\|\leq C\|\Delta(s)\|,$$
where $C< \infty$. 
Thus, by lemma \ref{lemma:gronwall}, we have 
$$\|\Delta(s)\|\leq e^{Ct}\|\Delta(0)\|=e^{Ct}\|\mathbf{u}(\mathbf{x}(0))-\mathbf{U}(0)\|=0.$$
This further implies that $\mathbf{U}(s)=\mathbf{u}(\mathbf{x}(s))$, so long as $\mathbf{a}$ and $\mathbf{c}$ are Lipschitz continuous.
\end{proof}

\subsection{Proof of Proposition \ref{prop:intersect}}
\label{proof:intersect}
\begin{proof}

Suppose have C-NODE given by
$$
\frac{\mathrm{d}u}{\mathrm{d}s} = \frac{\partial u}{\partial x}\frac{dx}{ds} + \frac{\partial u}{\partial t}\frac{dt}{ds}. 
$$
Write out specific functions for these terms to match the desired properties of the function.
Define initial condition $u(0,0)=u_0$. By setting
\begin{align*}
& \frac{dx}{ds}\left(s,u_0,\theta\right)=1, & &\frac{dt}{ds}\left(s,u_0,\theta\right)=u_0, \\
&\frac{\partial u}{\partial x}(u(x,t),\theta)=1, & &\frac{\partial u}{\partial t}(u(x,t),\theta)=-2,
\end{align*}
have the ODE and solution,
\begin{align*}
&\frac{\mathrm{d}u}{\mathrm{d}s} = 1 - 2u_0 \\
\implies& u(s;u_0) = \left ( 1 - 2u_0 \right )s \\
\implies& u\left(s; \begin{bmatrix}
0 \\
1
\end{bmatrix} \right) = \left ( 1 - 2\begin{bmatrix}
0 \\
1
\end{bmatrix} \right )s = \begin{bmatrix}
1 \\
-1
\end{bmatrix} s.
\end{align*}
To be specific, we can represent this system with the following family of PDEs:

\begin{align*}
\frac{\partial u}{\partial x}+u_0\frac{\partial u}{\partial t}=1-2u_0.
\end{align*}


\begin{figure}
    \centering
\includegraphics[width=0.45\textwidth,trim=20cm 4cm 18cm 5cm,clip]{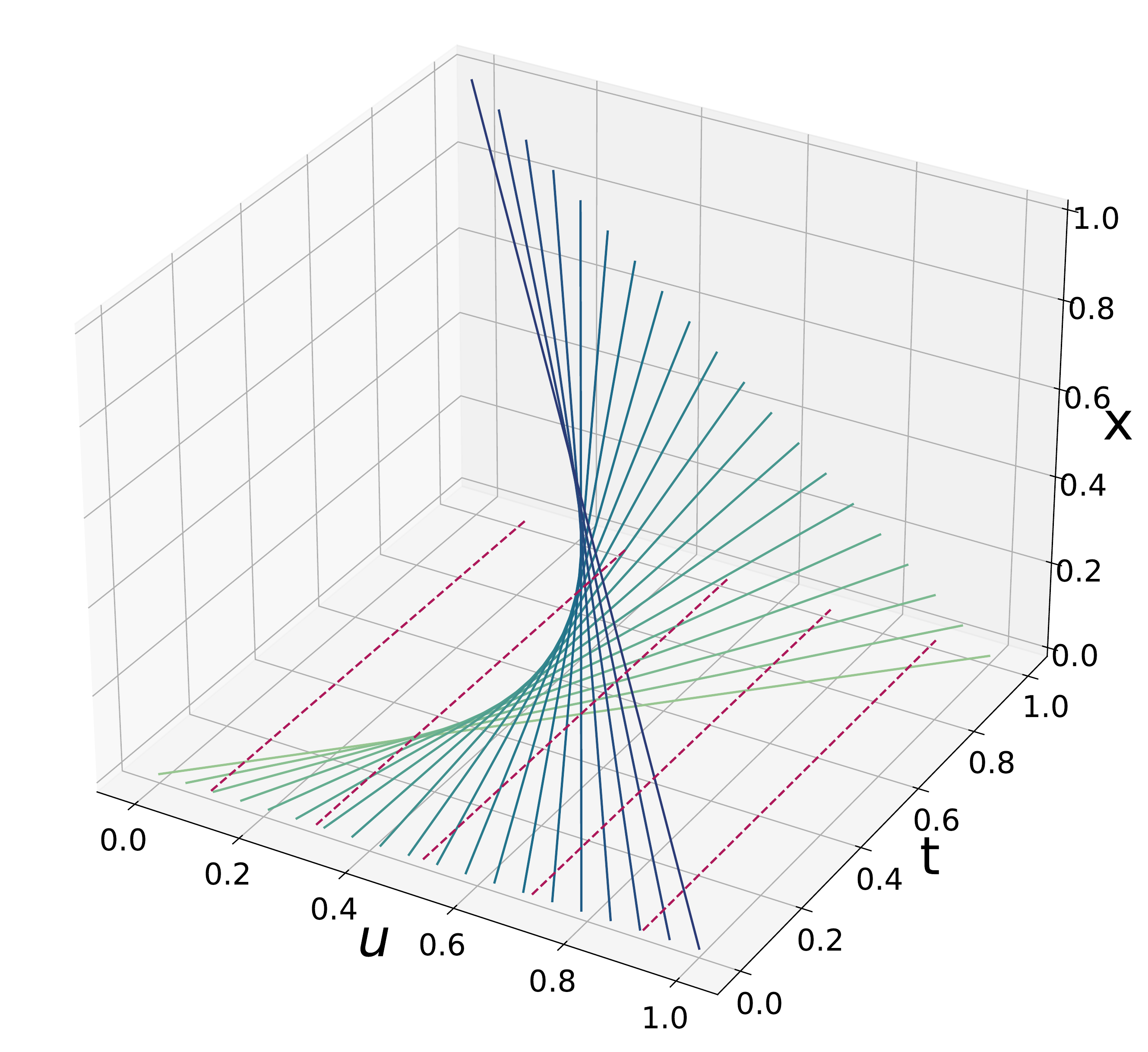}
    \caption{Comparison of C-NODEs and NODEs. C-NODEs (solid blue) learn a family of integration paths conditioned on the input value, avoiding intersecting dynamics. NODEs (dashed red) integrate along a 1D line that is not conditioned on the input value and can not represent functions requiring intersecting dynamics.}
    \label{fig:intersecting}
\end{figure}

We can solve this system to obtain a function that has intersecting trajectories.
The solution is visualized in Figure~\ref{fig:intersecting}, which shows that C-NODE can be used to learn and represent this function $\mathcal{G}$.
It should be noted that this is not the only possible solution to function $\mathcal{G}$, as when $\partial t/\partial s=0$, we fall back to a NODE system with the dynamical system conditioned on the input data. 
In this conditioned setting, we can then represent $\mathcal{G}$ by stopping the dynamics at different times $t$ as in \citet{massaroli2021dissecting}. 

\end{proof}

\subsection{Proof of Proposition \ref{prop:cnodelogprob}}
\label{proof:cnodelogprob}
The proof uses the change of variables formula for a particle that depends on a vector rather than a scalar and it follows directly from the proof given in~\citet[Appendix A]{chen2019neural}.
We provide the full proof for completeness. 

\begin{proof}
Assume $\sum_{i=1}^k\frac{\partial u}{\partial x_i}\frac{dx_i}{ds}$ is Lipschitz continuous in $u$ and continuous in $t$, so every initial value problem has a unique solution \citep{evans2010pde}. Also assume $u(s)$ is bounded. 

Want 
\begin{align*}
    \frac{\partial p(u(s))}{\partial s}=\tr\left (\frac{\partial}{\partial u}\sum_{i=1}^k\frac{\partial u}{\partial x_i}\frac{dx_i}{ds}\right).
\end{align*} 

Define $T_{\epsilon}=u(s+\epsilon)$.
The discrete change of variables states that $u_1=f(u_0)\Rightarrow \log p(u_1)=\log p(u_0)-\log |\det \frac{\partial f}{\partial u_0}|$ \citep{JimenezRezende2015VariationalIW}.

Take the limit of the time difference between $u_0$ and $u_1$, by definition of derivatives, 
\begin{align*}
    \frac{\partial \log p(u(s))}{\partial t}&=\lim_{\epsilon\rightarrow 0^+}\frac{\log p(u(s+\epsilon))-\log p(u(s))}{\epsilon}\\
    &=\lim_{\epsilon\rightarrow 0^+}\frac{\log p(u(s))-\log|\det \frac{\partial}{\partial u}T_{\epsilon}(u(t))|-\log p(u(s))}{\epsilon}\\
    &=-\lim_{\epsilon\rightarrow 0^+}\frac{\log|\det \frac{\partial}{\partial u}T_{\epsilon}(u(s))|}{\epsilon}\\
    &=-\lim_{\epsilon\rightarrow 0^+}\frac{\frac{\partial }{\partial \epsilon}\log|\det \frac{\partial}{\partial u}T_{\epsilon}(u(s))|}{\frac{\partial }{\partial \epsilon}\epsilon}\\
    &=-\lim_{\epsilon\rightarrow 0^+}\frac{\partial }{\partial \epsilon}\log|\det \frac{\partial}{\partial u}T_{\epsilon}(u(s))| -\lim_{\epsilon\rightarrow 0^+}\frac{\partial }{\partial \epsilon}\log|\det \frac{\partial}{\partial u}T_{\epsilon}(u(s))|\\
    &=-\lim_{\epsilon\rightarrow 0^+}\frac{1}{|\det \frac{\partial}{\partial u}T_{\epsilon}(u(s))|}\frac{\partial }{\partial \epsilon}|\det \frac{\partial}{\partial u}T_{\epsilon}(u(s))|\\
    &=-\frac{\lim_{\epsilon\rightarrow 0^+}\frac{\partial }{\partial \epsilon}|\det \frac{\partial}{\partial u}T_{\epsilon}(u(s))|}{\lim_{\epsilon\rightarrow 0^+}|\det \frac{\partial}{\partial u}T_{\epsilon}(u(s))|}\\
    &=-\lim_{\epsilon\rightarrow 0^+}\frac{\partial }{\partial \epsilon}|\det \frac{\partial}{\partial u}T_{\epsilon}(u(s))|\\
\end{align*}
The Jacobi's formula states that if $A$ is a differentiable map from the real numbers to $n\times n$ matrices, then $\frac{d}{dt}\det A(t)=tr(adj(A(t))\frac{d A(t)}{dt})$, where $adj$ is the adjugate. Thus, have
\begin{align*}
    \frac{\partial \log p(u(t))}{\partial t}&=-\lim_{\epsilon\rightarrow 0^+}\tr\left[\adj\left(\frac{\partial }{\partial u}T_{\epsilon}(u(s))\right)\frac{\partial }{\partial \epsilon}\frac{\partial }{\partial u}T_{\epsilon}(u(s))\right]\\
    &=-\tr\left[\left(\lim_{\epsilon\rightarrow 0^+}\adj\left(\frac{\partial}{\partial u}T_{\epsilon}(u(t))\right)\right)\left(\lim_{\epsilon\rightarrow 0^+}\frac{\partial}{\partial\epsilon}\frac{\partial}{\partial u}T_{\epsilon}(u(s))\right)\right]\\
    &=-\tr\left[\adj\left(\frac{\partial}{\partial u}u(t)\right)\lim_{\epsilon\rightarrow 0^+}\frac{\partial}{\partial\epsilon}\frac{\partial}{\partial u}T_{\epsilon}(u(s))\right]\\
    &=-\tr\left[\lim_{\epsilon\rightarrow 0^+}\frac{\partial}{\partial\epsilon}\frac{\partial}{\partial u}T_{\epsilon}(u(s))\right]\\
\end{align*}
    
Substituting $T_{\epsilon}$ with its Taylor series expansion and taking the limit, we have
\begin{align*}
    \frac{\partial \log p(u(t))}{\partial t}&=-\tr\left(\lim_{\epsilon\rightarrow 0^+}\frac{\partial}{\partial \epsilon}\frac{\partial}{\partial u}\left(u+\epsilon \frac{du}{ds}+\mathcal{O}(\epsilon^2)+\mathcal{O}(\epsilon^3)+...\right)\right)\\
    &=-\tr\left(\lim_{\epsilon\rightarrow 0^+}\frac{\partial}{\partial \epsilon}\frac{\partial}{\partial u}\left(u+\epsilon \sum_{i=1}^k\frac{\partial u}{\partial x_i}\frac{dx_i}{ds}+\mathcal{O}(\epsilon^2)+\mathcal{O}(\epsilon^3)+...\right)\right)\\
    &=-\tr\left(\lim_{\epsilon\rightarrow 0^+}\frac{\partial}{\partial \epsilon}\left(I+\frac{\partial }{\partial u}\epsilon \sum_{i=1}^k\frac{\partial u}{\partial x_i}\frac{dx_i}{ds}+\mathcal{O}(\epsilon^2)+\mathcal{O}(\epsilon^3)+...\right)\right)\\
    &=-\tr\left(\lim_{\epsilon\rightarrow 0^+}\left(\frac{\partial }{\partial u}\sum_{i=1}^k\frac{\partial u}{\partial x_i}\frac{dx_i}{ds}+\mathcal{O}(\epsilon)+\mathcal{O}(\epsilon^2)+...\right)\right)\\
    &=-\tr\left(\frac{\partial }{\partial u}\sum_{i=1}^k\frac{\partial u}{\partial x_i}\frac{dx_i}{ds}\right)\\
\end{align*}
\end{proof}

\subsection{Proof of Proposition \ref{prop:chomeo}}
\label{proof:chomeo}
\begin{proof}
To prove proposition \ref{prop:chomeo}, need to show that for any homeomorphism $h(\cdot)$, there exists a $u(s,u_0)\in\mathbb{R}^n$ following a C-NODE system such that $u(s=T,u_0)=h(u_0)$. 

Without loss of generality, say $T=1$.

Define C-NODE system 

\begin{align*}
    \begin{cases}
        \frac{du}{ds}=\frac{\partial u}{\partial x}\frac{dx}{ds}+\frac{\partial u}{\partial t}\frac{dt}{ds},\\
        \frac{dx}{ds}(s,u_0)=1,\\
        \frac{\partial u}{\partial x}(u(x,t))=h(u_0),\\
        \frac{dt}{ds}(s,u_0)=u_0,\\
        \frac{\partial u}{\partial t}(u(x,t))=-1.\\
    \end{cases}
\end{align*} 

Then, $\frac{du}{ds}=h(u_0)-u_0$. At $s=1$, have 
\begin{align*}
    u(s=1,u_0)&=u(s=0,u_0)+\int_0^1\frac{du}{ds}ds\\
    &=u_0+\int_0^1\frac{\partial u}{\partial x}\frac{dx}{ds}+\frac{\partial u}{\partial t}\frac{dt}{ds}ds\\
    &=u_0+\int_0^1h(u_0)\cdot 1+(-1)\cdot u_0ds\\
    &=u_0+h(u_0)-u_0\\
    &=h(u_0).
\end{align*}  

The inverse map will be defined by integration backwards. Specifically, have 

\begin{align*}
    u(s=0,u_0)&=u(s=1,u_0)+\int_1^0\frac{du}{ds}ds\\
    &=h(u_0)-\int_0^1\frac{\partial u}{\partial x}\frac{dx}{ds}+\frac{\partial u}{\partial t}\frac{dt}{ds}ds\\
    &=h(u_0)-\int_0^1h(u_0)\cdot 1+(-1)\cdot u_0ds\\
    &=h(u_0)-h(u_0)+u_0\\
    &=u_0.
\end{align*}  

Thus, for any homeomorphism $h(\cdot)$, there exists a C-NODE system, such that forward integration for time $s=1$ is equivalent as applying $h(\cdot)$, and backward integration for time $s=1$ is equivalent to applying $h^{-1}(\cdot)$.
\end{proof}

\section{Ablation Study}
\label{sec:ablation}

In previous experiments, we represent $\partial \mathbf{u}/\partial x_i$ with separate and independent neural networks $\mathbf{c}_i(\mathbf{u},\theta)$. Here, we represent all $k$ functions as a vector-valued function $[\partial \mathbf{u}/\partial x_1,...,\partial \mathbf{u}/\partial x_k]^T$. We approximate this vector-valued function with a neural network $\mathbf{c}(\mathbf{u},\theta)$. The model is trained using the Euler solver to have better training stability when the neural network has a large number of parameters. Experiment details for the ablation study is as shown in Figures \ref{fig:ablation_mnist}, \ref{fig:ablation_svhn}, \ref{fig:ablation_cifar}.

\begin{figure*}[hbt!]
    \centering
    \includegraphics[width=0.9\textwidth,trim=0cm 0cm 0cm 0cm,clip]{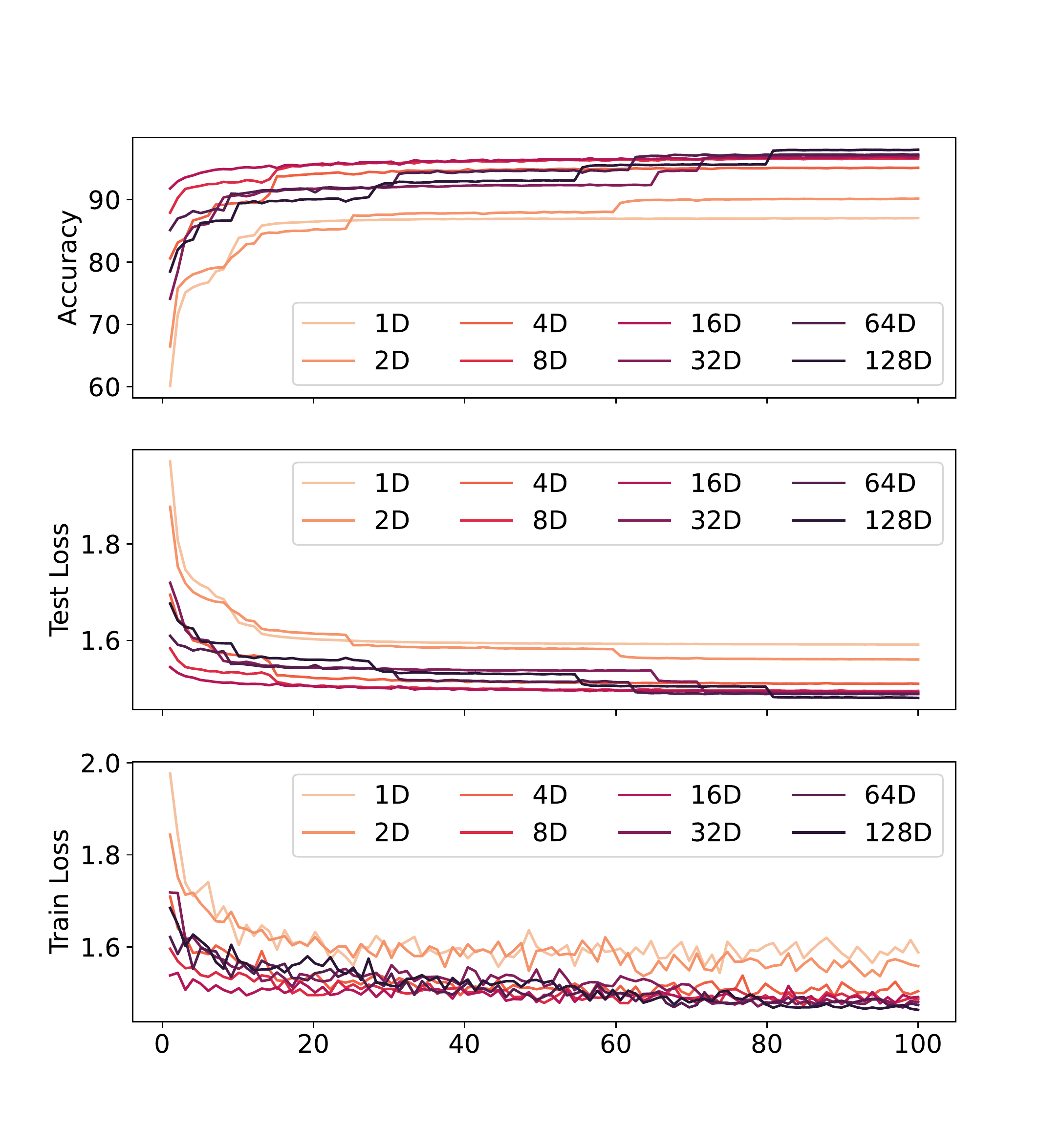}
    \caption{The training process averaged over 4 runs of C-NODE with 1, 2, 4, 8, 16, 32, 64, 128, 256, 512, and 1024 dimensions on the MNIST dataset. The first row is the accuracy of prediction, the second row is the testing error, and the third row is the training error.}
    \label{fig:ablation_mnist}
\end{figure*}

\begin{figure*}[hbt!]
    \centering
    \includegraphics[width=0.9\textwidth,trim=0cm 0cm 0cm 0cm,clip]{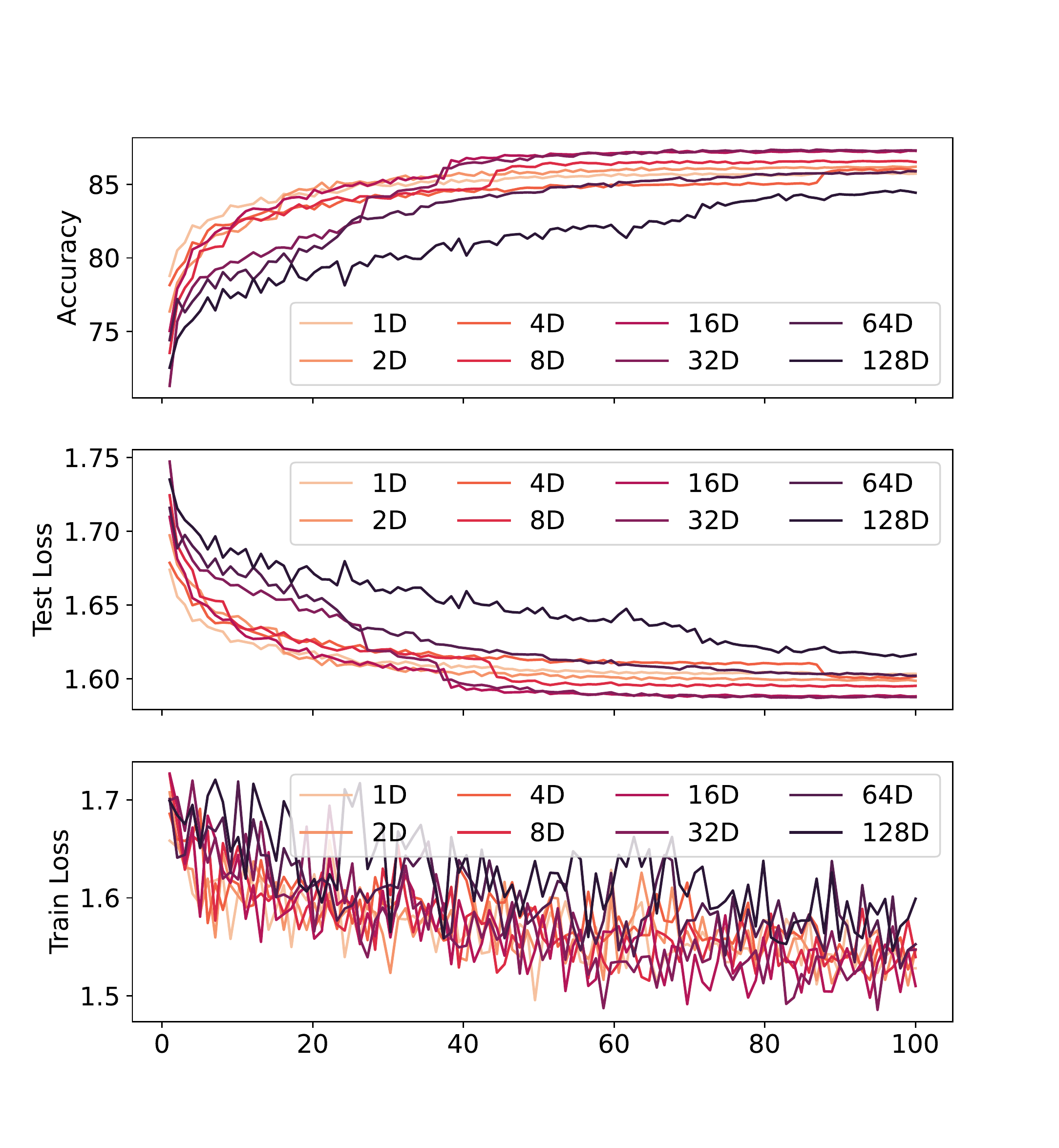}
    \caption{The training process averaged over 4 runs of C-NODE with 1, 2, 4, 8, 16, 32, 64, and 128 dimensions on the SVHN dataset. The first row is the accuracy of prediction, the second row is the testing error, and the third row is the training error.}
    \label{fig:ablation_svhn}
\end{figure*}

\begin{figure*}[hbt!]
    \centering
    \includegraphics[width=0.9\textwidth,trim=0cm 0cm 0cm 0cm,clip]{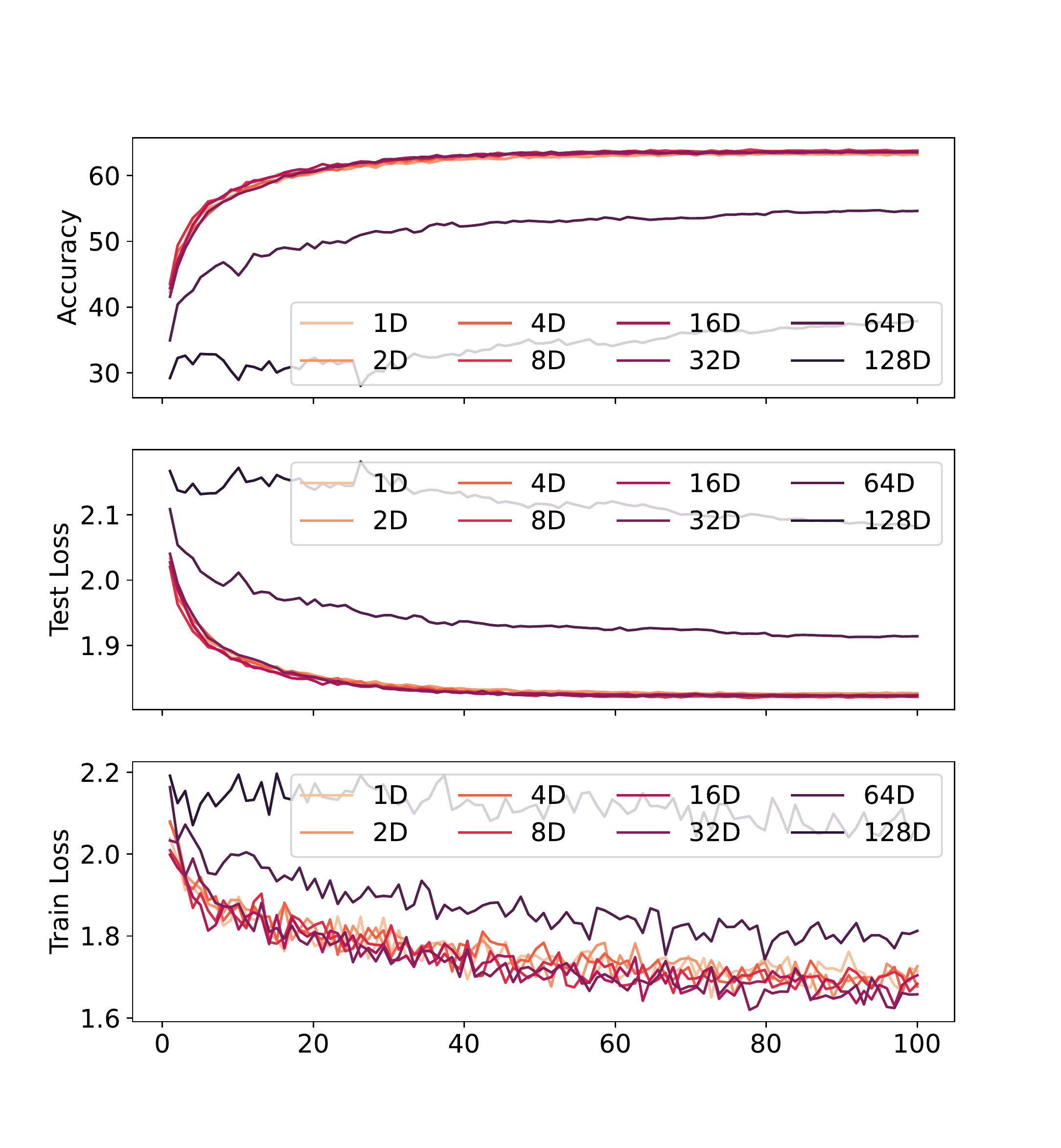}
    \caption{The training process averaged over 4 runs of C-NODE with 1, 2, 4, 8, 16, 32, 64, and 128 dimensions on the CIFAR-10 dataset. The first row is the accuracy of prediction, the second row is the testing error, and the third row is the training error.}
    \label{fig:ablation_cifar}
\end{figure*}

\end{document}